\documentclass{article}

\usepackage{microtype}
\usepackage{graphicx}
\usepackage{subcaption}
\usepackage{booktabs} 
\usepackage{amsmath}
\usepackage{amsthm}
\usepackage{amssymb}
\usepackage{amsfonts}

\graphicspath{{figures/}}

\usepackage{tikz,pgf}
\usetikzlibrary{arrows,automata,shapes,calc,backgrounds,spy,positioning,shadows,shapes.misc}

\newcommand{\cA}{\mathcal{A}}

\newcommand{\cS}{\mathcal{S}}
\newcommand{\cM}{\mathcal{M}}
\newcommand{\cI}{\mathcal{I}}
\newcommand{\cW}{\mathcal{W}}
\newcommand{\cF}{\mathcal{F}}
\newcommand{\cP}{\mathcal{P}}
\newcommand{\cT}{\mathcal{T}}
\newcommand{\cR}{\mathcal{R}}
\newcommand{\cO}{\mathcal{O}}
\newcommand{\cE}{\mathcal{E}}

\newtheorem{theorem}{Theorem}[section]

\newtheorem{definition}{Definition}[section]
\newtheorem{lemma}[theorem]{Lemma}

\newtheorem{condition}[theorem]{Condition}

\newcommand{\CA}[1]{\mathcal{#1}}

\newcommand{\Next}{\bigcirc}
\newcommand{\Always}{\Box}
\newcommand{\Event}{\diamondsuit}
\DeclareMathOperator{\until}{\CA{U}}

\DeclareMathOperator*{\argmax}{arg\,max}

\usepackage{hyperref}

\newlength\myindent
\usepackage[accepted]{icml2021}

\icmltitlerunning{The Logical Options Framework}

\begin{document}

\twocolumn[
\icmltitle{The Logical Options Framework}



\icmlsetsymbol{equal}{*}

\begin{icmlauthorlist}
\icmlauthor{Brandon Araki}{mit}
\icmlauthor{Xiao Li}{mit}
\icmlauthor{Kiran Vodrahalli}{col}
\icmlauthor{Jonathan DeCastro}{tri}
\icmlauthor{J. Micah Fry}{lin}
\icmlauthor{Daniela Rus}{mit}
\end{icmlauthorlist}

\icmlaffiliation{mit}{CSAIL, Massachusetts Institute of Technology, Cambridge, MA}
\icmlaffiliation{col}{Department of Computer Science, Colombia University, New York City, NY}
\icmlaffiliation{lin}{MIT Lincoln Laboratory, Lexington, MA}
\icmlaffiliation{tri}{Toyota Research Institute, Cambridge, MA}

\icmlcorrespondingauthor{Brandon Araki}{araki@mit.edu}

\icmlkeywords{Hierarchical Reinforcement Learning, Options Framework, Formal Logic, Planning}

\vskip 0.3in
]



\printAffiliationsAndNotice{}  

\begin{abstract}
Learning composable policies for environments with complex rules and tasks is a challenging problem. We introduce a hierarchical reinforcement learning framework called the \textit{Logical Options Framework} (LOF) that learns policies that are \textit{satisfying}, \textit{optimal}, and \textit{composable}. LOF efficiently learns policies that satisfy tasks by representing the task as an automaton and integrating it into learning and planning. We provide and prove conditions under which LOF will learn satisfying, optimal policies. And lastly, we show how LOF's learned policies can be composed to satisfy unseen tasks with only 10-50 retraining steps. We evaluate LOF on four tasks in discrete and continuous domains, including a 3D pick-and-place environment.
\end{abstract}

\section{Introduction}

To operate in the real world, intelligent agents must be able to make long-term plans by reasoning over symbolic abstractions while also maintaining the ability to react to low-level stimuli in their environment \citep{zhang2020survey}. Many environments obey rules that can be represented as logical formulae; e.g., the rules a driver follows while driving, or a recipe a chef follows to cook a dish. Traditional motion and path planning techniques struggle to plan over these long-horizon tasks, but hierarchical approaches such as hierarchical reinforcement learning (HRL) can solve lengthy tasks by planning over both the high-level rules and the low-level environment. However, solving these problems involves trade-offs among multiple desirable properties, which we identify as \textit{satisfaction}, \textit{optimality}, and \textit{composability} (described below). Today's hierarchical planning algorithms lack at least one of these objectives. For example, Reward Machines \citep{icarte2018using} are satisfying and optimal, but not composable; the options framework \citep{sutton1999between} is composable and hierarchically optimal, but cannot satisfy specifications. An algorithm that achieves all three of these properties would be very powerful because it would enable a model learned on one set of rules to generalize to arbitrary rules. We introduce the \textit{Logical Options Framework}, which builds upon the options framework and aims to combine symbolic reasoning and low-level control to achieve satisfaction, optimality, and composability with as few compromises as possible. Furthermore, we demonstrate that models learned with our framework generalize to arbitrary sets of rules without any further learning, and we also show that our framework is compatible with arbitrary domains and planning algorithms, from discrete domains and value iteration to continuous domains and proximal policy optimization (PPO).

\textbf{Satisfaction:} An agent operating in an environment governed by rules must be able to satisfy the specified rules. Satisfaction is a concept from formal logic, in which the input to a logical formula causes the formula to evaluate to \texttt{True}. Logical formulae can encapsulate rules and tasks like the ones described in Fig.~\ref{fig:demo-domain}, such as ``pick up the groceries'' and ``do not drive into a lake''. In this paper, we state conditions under which our method is guaranteed to learn satisfying policies.

\begin{figure*}[!th]
\centering
\begin{subfigure}[b]{0.85\textwidth}
  \centering
  \includegraphics[width=0.99\textwidth]{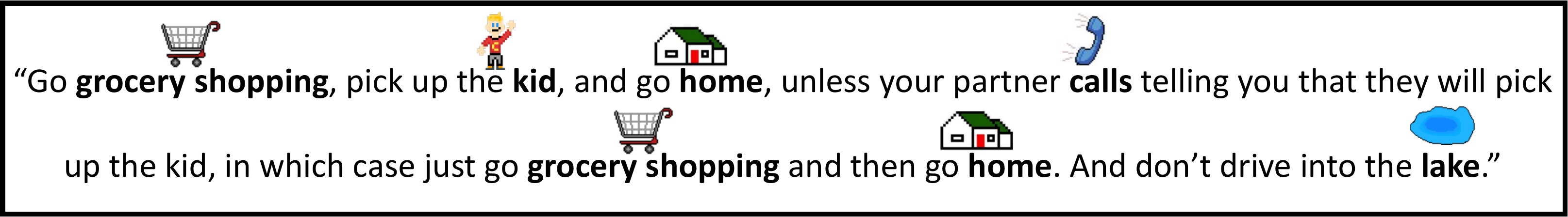}
  \caption{These natural language instructions can be transformed into an FSA, shown in (b).}
  \label{fig:demo-nl}
\end{subfigure}

\begin{subfigure}[t]{0.49\textwidth}
  \centering
  \includegraphics[width=.8\textwidth]{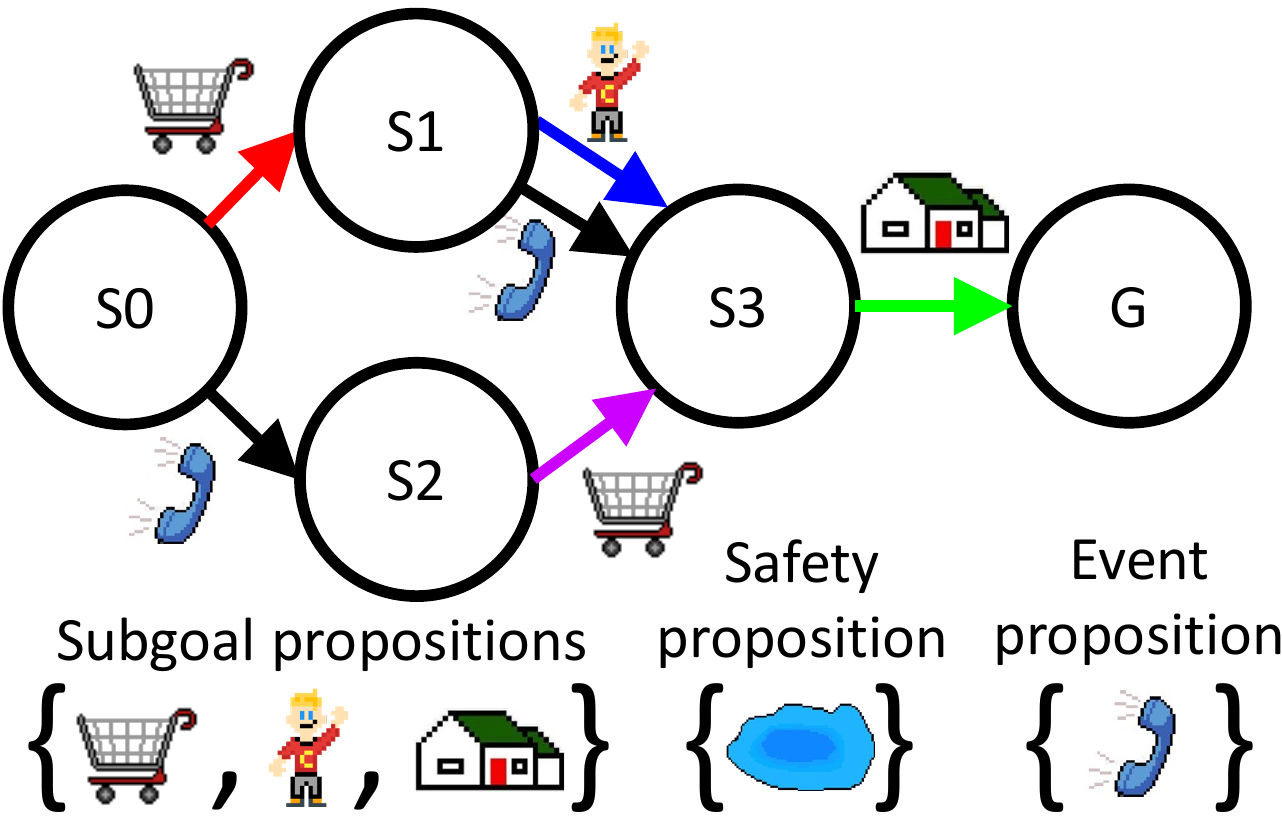}
  \caption{The FSA representing the natural language instructions. The propositions are divided into ``subgoal'', ``safety'', and ``event.''}
  \label{fig:demo-ltl-fsa}
\end{subfigure} \hfill
\begin{subfigure}[t]{.44\textwidth}
  \centering
  \includegraphics[width=.6\textwidth]{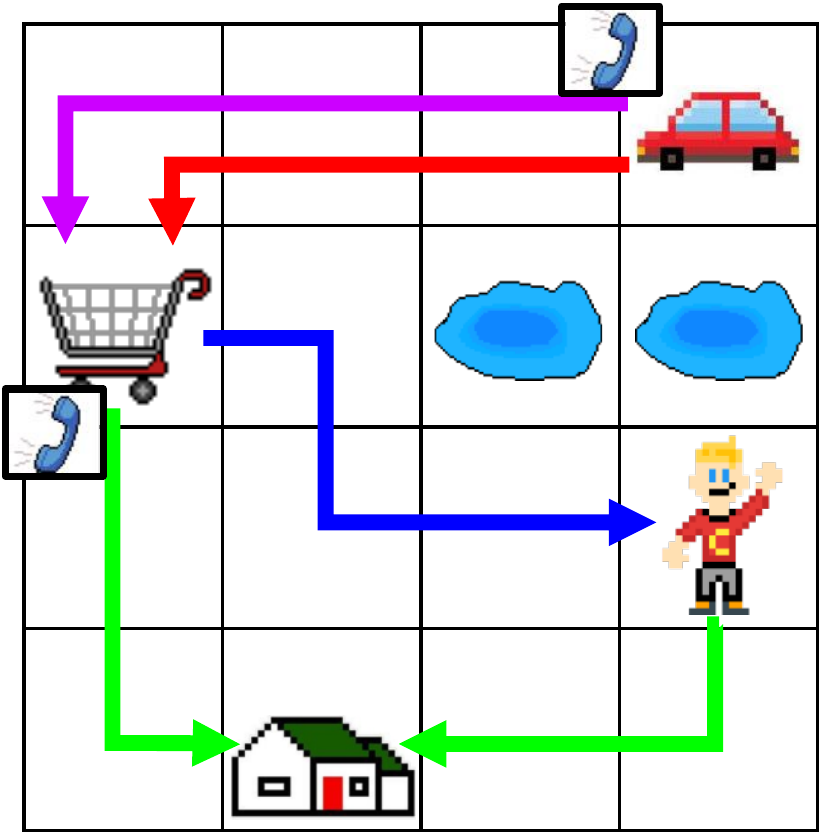}
  \caption{The low-level MDP and corresponding policy that satisfies the instructions.}
  \label{fig:demo-mdp}
\end{subfigure}

\caption{Many parents face this task after school ends -- who picks up the kid, and who gets groceries? The pictorial symbols represent propositions, which are true or false depending on the state of the environment.
The arrows in (c) represent sub-policies, and the colors of the arrows match the corresponding transition in the FSA. The boxed phone at the beginning of some of the arrows represents how these sub-policies can occur only after the agent receives a phone call.}
\label{fig:demo-domain}
\end{figure*}

\textbf{Optimality:} Optimality requires that the agent maximize its expected cumulative reward for each episode. In general, satisfaction can be achieved by rewarding the agent for satisfying the rules of the environment. In hierarchical planning there are several types of optimality, including hierarchical optimality (optimal with respect to the hierarchy) and optimality (optimal with respect to everything). We prove in this paper that our method is hierarchically optimal and, under certain conditions, optimal.

\textbf{Composability:} Our method is also composable -- once it has learned the low-level components of a task, the learned model can be rearranged to satisfy arbitrary tasks. More specifically, the rules of an environment can be factored into liveness and safety properties, which we discuss in Sec.~\ref{sec:lof}. The learned model has high-level actions called options that can be composed to satisfy new liveness properties. A shortcoming of many RL models is that they are not composable -- trained to solve one specific task, they are incapable of handling even small variations in the task structure. However, the real world is a dynamic and unpredictable place, so the ability to use a learned model to automatically reason over as-yet-unseen tasks is a crucial element of intelligence.

Fig.~\ref{fig:demo-domain} gives an example of how LOF works. The environment is a world with a grocery store, your (hypothetical) kid, your house, and some lakes, and in which you, the agent, are driving a car. The propositions are divided into ``subgoals'', representing events that can be achieved, such as going grocery shopping; ``safety'' propositions, representing events that you must avoid (driving into a lake); and ``event'' propositions, corresponding to events that you have no control over (receiving a phone call) (Fig.~\ref{fig:demo-ltl-fsa}). In this environment, you have to follow rules (Fig.~\ref{fig:demo-nl}). These rules can be converted into a logical formula, and from there into a finite state automaton (FSA) (Fig.~\ref{fig:demo-ltl-fsa}). LOF learns an option for each subgoal (illustrated by the arrows in Fig.~\ref{fig:demo-mdp}), and a meta-policy for choosing amongst the options to reach the goal state of the FSA. After learning, the options can be recombined to fulfill arbitrary tasks.

\subsection{Contributions}

This paper introduces the Logical Options Framework (LOF) and makes four contributions to the hierarchical reinforcement learning literature:

\begin{enumerate}
\item The definition of a hierarchical semi-Markov Decision Process (SMDP) that is the product of a logical FSA and a low-level environment MDP.
\item A planning algorithm for learning options and meta-policies for the SMDP that allows the options to be composed to solve new tasks with only 10-50 retraining steps and no additional samples from the environment.
\item Conditions and proofs for satisfaction and optimality.
\item Experiments on a discrete delivery domain, a continuous 2D reacher domain, and a continuous 3D pick-and-place domain on four tasks demonstrating satisfaction, optimality, and composability.
\end{enumerate}


\section{Background}



\textbf{Linear Temporal Logic:} We use linear temporal logic (LTL) to formally specify rules \citep{Clark01}. LTL can express tasks and rules using temporal operators such as ``eventually'' and ``always.'' LTL formulae are used only indirectly in LOF, as they are converted into automata that the algorithm uses directly. We chose to use LTL to represent rules because LTL corresponds closely to natural language and has proven to be a more natural way of expressing tasks and rules for engineers than designing FSAs by hand \citep{kansou2019converting}. Formulae $\phi$ 
have the syntax grammar
\begin{align*}
\phi := p \;|\; \neg \phi \;|\; \phi_1 \vee \phi_2 \;|\; \Next \phi \;|\; \phi_1 \until \phi_2 
\end{align*}
where $p$ is a \textit{proposition} (a boolean-valued truth statement that can correspond to objects or events in the world), $\neg$ is negation, $\vee$ is disjunction, $\Next$ is ``next'', and $\until$ is ``until''. The derived rules are conjunction $(\wedge)$, implication $(\implies)$, equivalence $(\leftrightarrow)$, ``eventually'' ($\Event \phi \equiv \texttt{True} \until \phi$) and ``always'' ($\Always \phi \equiv \neg \Event \neg \phi$) \citep{Baier08}. $\phi_1 \until \phi_2$ means that $\phi_1$ is true until $\phi_2$ is true, $\Event \phi$ means that there is a time where $\phi$ is true and $\Always \phi$ means that $\phi$ is always true.

\textbf{The Options Framework:} The options framework is a framework for defining and solving semi-Markov Decision Processes (SMDPs) with a type of macro-action called an option \citep{sutton1999between}.
The inclusion of options in an MDP problem turns it into an SMDP problem, because actions are dependent not just on the previous state but also on the identity of the currently active option, which could have been initiated many time steps before the current time.

An option $o$ is a variable-length sequence of actions defined as $o = (\cI, \pi, \beta, R_o(s), T_o(s' \vert s))$. $\cI \subseteq \cS$ is the initiation set of the option. $\pi : \cS \times \cA \rightarrow [0, 1]$ is the policy of the option. $\beta : \cS \rightarrow [0, 1]$ is the termination condition. $R_o(s)$ is the reward model of the option. $T_o(s' \vert s)$ is the transition model. A major challenge in option learning is that, in general, the number of time steps before the option terminates, $k$, is a random variable. With this in mind, $R_o(s)$ is defined as the expected cumulative reward of option $o$ given that the option is initiated in state $s$ at time $t$ and ends after $k$ time steps. Letting $r_t$ be the reward received by the agent at $t$ time steps from the beginning of the option, 

\begin{equation}
\label{eq:option-reward}
R_o(s) = \mathbb{E} \big[ r_{1} + \gamma r_{2} + \dots \gamma^{k-1}r_{k} \big]
\end{equation}

$T_o(s' \vert s)$ is the combined probability $p_o(s', k)$ that option $o$ will terminate at state $s'$ after $k$ time steps: 

\begin{equation}
\label{eq:option-transitions}
T_o(s' \vert s) = \sum_{k=1}^\infty p_o(s', k)\gamma^k
\end{equation}

A crucial benefit of using options is that they can be composed in arbitrary ways. In the next section, we describe how LOF composes them to satisfy logical specifications.




\section{Logical Options Framework}\label{sec:lof}

Here is a brief overview of how we will present our formulation of LOF:

\begin{enumerate}
    \item The LTL formula is decomposed into liveness and safety properties. The liveness property defines the task specification and the safety property defines the costs for violating rules.
    \item The propositions are divided into subgoals, safety propositions, and event propositions. Each subgoal is associated with its own option, whose goal is to achieve that subgoal. Safety propositions are used to define rules. Event propositions serve as control flow variables that affect the task.
    \item We define an SMDP that is the product of a low-level MDP and a high-level logical FSA.
    \item We define the logical options.
    \item We present an algorithm for finding the hierarchically optimal policy on the SMDP.
    \item We state conditions under which satisfaction of the LTL specification is guaranteed, and we prove that the planning algorithm converges to an optimal policy by showing that the hierarchically optimal SMDP policy is the same as the optimal MDP policy.
\end{enumerate}



\textbf{The Logic Formula:} LTL formulae can be translated into B\"uchi automata using automatic translation tools such as SPOT \citep{Duret16}. All B\"uchi automata can be decomposed into liveness and safety properties \citep{alpern1987recognizing}. We assume here that the LTL formula itself can be divided into liveness and safety formulae, $\phi = \phi_{liveness} \land \phi_{safety}$. For the case where the LTL formula cannot be factored, see App.~\ref{sec:appendix-lof}.
The liveness property describes ``things that must happen'' to satisfy the LTL formula. It is a task specification and is used in planning to determine which subgoals the agent must achieve. The safety property describes ``things that can never happen'' and is used to define costs for violating the rules. In LOF, the liveness property is written using a finite-trace subset of LTL called syntactically co-safe LTL \citep{bhatia2010sampling}, in which $\Always$ (``always'') is not allowed and $\Next$, $\until$, and $\Event$ are only used in positive normal form. This way, the liveness property can be satisfied by finite sequences of propositions, so the property can be represented as an FSA.

\textbf{Propositions:} Propositions are boolean-valued truth statements corresponding to goals, objects, and events in the environment. 
We distinguish between three types of propositions: subgoals $\cP_G$, safety propositions $\cP_S$, and event propositions $\cP_E$.
Subgoals must be achieved in order to satisfy the liveness property. They are associated with goals such as ``the agent is at the grocery store''. They only appear in $\phi_{liveness}$. Each subgoal may only be associated with one state. Note that in general, it may be impossible to avoid having subgoals appear in $\phi_{safety}$. App.~\ref{sec:appendix-lof} describes how to deal with this scenario.
Safety propositions are propositions that the agent must avoid -- for example, driving into a lake. They only appear in $\phi_{safety}$.
Event propositions are not goals, but they can affect the task specification -- for example, whether or not a phone call is received. They may occur in $\phi_{liveness}$, and, with extensions described in App.~\ref{sec:appendix-lof}, in $\phi_{safety}$. In the fully observable setting, event propositions are somewhat trivial because the agent knows exactly when/if the event will occur, but in the partially observable setting, they enable complex control flow.
Our optimality guarantees only apply in the fully observable setting; however, LOF's properties of satisfaction and composability still apply in the partially observable setting. The goal state of the liveness FSA must be reachable from every other state using only subgoals. This means that no matter what event propositions occur, it must be possible for the agent to satisfy the liveness property. 
$T_{P_G} : \cS \rightarrow 2^{\cP_G}$ and $T_{P_S} : \cS \rightarrow 2^{\cP_S}$ relate states to the subgoal and safety propositions that are true at that state. $T_{P_E} : 2^{\cP_E} \rightarrow \{0, 1\}$ assigns truth labels to the event propositions.

\begin{algorithm}[!t]
\caption{Learning and Planning with Logical Options}\label{alg:lof}
\begin{algorithmic}[1]
    \STATE \textbf{Given:} \par
        Propositions $\cP$ partitioned into subgoals $\cP_G$, safety propositions $\cP_S$, and event propositions $\cP_E$ \par
        Logical FSA $\cT = (\cF, \cP_G \times \cP_E, T_F, R_F, f_0, f_g)$ derived from $\phi_{liveness}$ \par
        Low-level MDP $\cE = (\cS, \cA, R_\cE, T_E, \gamma)$, where $R_\cE(s, a) = R_E(s, a) + R_S(T_{P_S}(s))$ combines the environment and safety rewards \par
        Proposition labeling functions $T_{P_G} : \cS \rightarrow 2^{\cP_G}$,  $T_{P_S} : \cS \rightarrow 2^{\cP_S}$, and $T_{P_E} : 2^{\cP_E} \rightarrow \{0, 1\}$ \par
    \STATE \textbf{To learn:} 
        \STATE Set of options $\cO$, one for each subgoal $p \in \cP_G$
        \STATE Meta-policy $\mu(f, s, o)$, $Q(f, s, o)$, and $V(f, s)$
    \STATE \textbf{Learn logical options:}\label{alg:logical-options1}
    \FOR{ $p \in \cP_G$}
    \STATE Learn an option that achieves $p$, \par $o_p = (\cI_{o_p}, \pi_{o_p}, \beta_{o_p}, R_{o_p}(s), T_{o_p}(s' \vert s))$
        \STATE $ \cI_{o_p} = \cS $
        \STATE $ \beta_{o_p} = \begin{cases} 1 &\quad\text{if } p \in T_{P_G}(s)\\
                                             0 &\quad\text{otherwise} \\
                                             \end{cases} $
        \STATE  $\pi_{o_p} = $ optimal policy on $\cE$ with rollouts terminating when $p \in T_{P_G}(s)$
        \STATE  $T_{o_p}(s' \vert s) = \begin{cases} \mathbb{E} \gamma^k &\quad \text{\parbox{11em}{if $p \in T_{P_G}(s')$; $k$ is number of time steps to reach $p$}}\\
                                             0 &\quad\text{otherwise} \\
                                             \end{cases}$\label{alg:lof-transition-model}
        \STATE $R_{o_p}(s) = \mathbb{E} [ R_\cE(s, a_{1}) + \gamma R_\cE(s_{1}, a_{2}) + \dots$ \par 
        \quad\quad\quad\quad\;\;\;$ + \gamma^{k-1} R_\cE(s_{k-1}, a_{k}) ]$ \label{alg:logical-options2}
    \ENDFOR
    \STATE \textbf{Find a meta-policy $\mu$ over the options:} \label{alg:lof-metapolicy} 
        \STATE Initialize $Q : \cF \times \cS \times \cO \rightarrow \mathbb{R}$,  $V : \cF \times \cS \rightarrow \mathbb{R}$ to $0$ 
        \FOR {$(k, f, s) \in [1, \dots, n] \times \cF \times \cS$}
        \FOR{$o \in \cO$}
        \STATE $Q_k(f, s, o) \leftarrow R_F(f)R_o(s) +$ \par $ \sum\limits_{f' \in \cF} \sum\limits_{\bar{p}_e \in 2^{\cP_E}} \sum\limits_{s' \in \cS} T_F(f' \vert f, T_P(s'), \bar{p}_e) T_{P_E}(\bar{p}_e)$ \par 
        \quad\quad\quad\quad\quad\quad\quad$T_o(s' \vert s) V_{k-1}(f', s')$
        \ENDFOR
        \STATE $V_k(f, s) \leftarrow \max\limits_{o \in \cO} Q_k(f, s, o)$
        \ENDFOR
    \STATE $\mu(f, s, o) = \argmax\limits_{o \in \cO} Q(f, s, o)$
    \STATE \textbf{Return:} Options $\cO$, meta-policy $\mu(f, s, o)$ and Q- and value functions $Q(f, s, o), V(f, s)$
\end{algorithmic}
\end{algorithm}

\textbf{Hierarchical SMDP:} LOF defines a hierarchical semi-Markov Decision Process (SMDP), learns the options, and plans over them. The high level of the SMDP is an FSA specified with LTL. The low level is an environment MDP.
We assume that the LTL specification $\phi$ can be decomposed into a liveness property $\phi_{liveness}$ and a safety property $\phi_{safety}$. The propositions $\cP$ are the union of the subgoals $\cP_G$, safety propositions $\cP_S$, and event propositions $\cP_E$. We assume that the liveness property can be translated into an FSA $\cT = (\cF, \cP, T_F, R_F, f_0, f_g)$. $\cF$ is the set of automaton states; $\cP$ is the set of propositions; $T_F$ is the transition function relating the current state and proposition to the next state, $T_F: \cF \times \cP \times \cF \rightarrow [0, 1]$. In practice, $T_F$ is deterministic despite our use of probabilistic notation. We assume that there is a single initial state $f_0$ and final state $f_g$, and that the goal state $f_g$ is reachable from every state $f \in \cF$ using only subgoals. The reward function assigns a reward to every FSA state, $R_F : \cF \rightarrow \mathbb{R}$. In our experiments, the safety property takes the form $\bigwedge_{p_s \in \cP_S} \Always \neg p_s$, which implies that no safety proposition is allowed, and that they have associated costs, $R_S : 2^{\cP_S} \rightarrow \mathbb{R}$. $\phi_{safety}$ is not limited to this form; App.~\ref{sec:appendix-lof} covers the general case.
There is a low-level environment MDP $\cE = (\cS, \cA, R_\cE, T_E, \gamma)$. $\cS$ is the state space and $\cA$ is the action space. They can be discrete or continuous. $R_E: \cS \times \cA \rightarrow \mathbb{R}$ is a low-level reward function that characterizes, for example, distance or actuation costs. $R_\cE$ is a combination of the safety reward function $R_S$ and $R_E$, e.g. $R_\cE(s, a) = R_E(s, a) + R_S(T_{P_S}(s))$.
The transition function of the environment is $T_E : \cS \times \cA \times \cS \rightarrow [0 ,1]$.

From these parts we define a hierarchical SMDP $\cM = (\cS \times \cF, \cA, \cP, \cO, T_E \times T_P \times T_F, R_{SMDP}, \gamma)$. The hierarchical state space contains two elements: low-level states $\cS$ and FSA states $\cF$. The action space is $\cA$. The set of propositions is $\cP$. The set of options (one option associated with each subgoal in $\cP_G$) is $\cO$. The transition function consists of the low-level environment transitions $T_E$ and the FSA transitions $T_F$. $T_P = T_{P_G} \times T_{P_S} \times T_{P_E}$. We call $T_P$, relating states to propositions, a transition function because it determines when FSA transitions occur. The transitions are applied in the order $T_E$, $T_P$, $T_F$. The reward function $R_{SMDP}(f, s, o) = R_F(f)R_o(s)$, so $R_F(f)$ is a weighting on the option rewards. The SMDP has the same discount factor $\gamma$ as $\cE$. Planning is done on the SMDP in two steps: first, the options $\cO$ are learned over $\cE$ using an appropriate policy-learning algorithm such as PPO or Reward Machines. Next, a meta-policy over the task specification $\cT$ is found using the learned options and the reward function $R_{SMDP}$.

\textbf{Logical Options:} The first step of Alg.~\ref{alg:lof} is to learn the logical options. We associate every subgoal $p$ with an option $o_{p} = (\cI_{o_p}, \pi_{o_p}, \beta_{o_p}, R_{o_p}, T_{o_p})$. These terms are defined starting at Alg.~\ref{alg:lof} line~\ref{alg:logical-options1}. Every $o_{p}$ has a policy $\pi_{o_p}$ whose goal is to reach the state $s_{p}$ where $p$ is true. Options are learned by training on the environment MDP $\cE$ and terminating only when $s_{p}$ is reached. As we discuss in Sec.~\ref{sec:proofs}, under certain conditions the optimal option policy is guaranteed to always terminate at the subgoal. This allows us to simplify the transition model of Eq.~\ref{eq:option-transitions} to the form in Alg.~\ref{alg:lof} line~\ref{alg:lof-transition-model}. In the experiments, we further simplify this expression by setting $\gamma = 1$.

\textbf{Logical Value Iteration:} After finding the logical options, the next step is to find a meta-policy for FSA $\cT$ over the options (see Alg.~\ref{alg:lof} line~\ref{alg:lof-metapolicy}). Q- and value functions are found for the SMDP using the Bellman update equations:

\begin{align}
\begin{split}\label{eq:q-update}
Q_k(& f, s, o) \leftarrow R_F(f)R_o(s) +  \sum_{f' \in \cF} \sum_{\bar{p}_e \in 2^{\cP_E}} \sum_{s' \in \cS} \\
& T_F(f' \vert f, T_{P_G}(s'), \bar{p}_e) T_{P_E}(\bar{p}_e ) T_o(s' \vert s) V_{k-1}(f', s')
\end{split}\\
V_k&( f, s) \leftarrow \max_{o \in \cO} Q_k(f, s, o)  \label{eq:v-update}
\end{align}


Eq.~\ref{eq:q-update} differs from the generic equations for SMDP value iteration in that the transition function has two extra components, $\sum_{f' \in \cF} T_F(f' \vert f, T_P(s'), \bar{p}_e)$ and $\sum_{\bar{p}_e \in 2^{\cP_E}} T_{P_E}(\bar{p}_e)$. The equations are derived from \citet{araki2019learning} and the fact that, on every step in the environment, three transitions are applied: the option transition $T_o$, the event proposition ``transition'' $T_{P_E}$, and the FSA transition $T_F$. Note that $R_o(s)$ and $T_o(s' \vert s)$ compress the consequences of choosing an option $o$ at a state $s$ from a multi-step trajectory into two real-valued numbers, allowing for more efficient planning.

\subsection{Conditions for Satisfaction and Optimality}\label{sec:proofs}

Here we give an overview of the proofs and necessary conditions for satisfaction and optimality. The full proofs and definitions are in App.~\ref{sec:appendix-proofs}.

First, we describe the condition for an optimal option to always reach its subgoal. Let $\pi'(s \vert s')$ be the optimal goal-conditioned policy for reaching a goal $s'$. If the optimal option policy equals the goal-conditioned policy for reaching the subgoal $s_g$, i.e. $\pi^*(s) = \pi_g(s \vert s_g)$, then the option will always reach the subgoal. This can be stated in terms of value functions: let $V^{\pi'}(s \vert s')$ be the expected return of $\pi'(s \vert s')$. If $V^{\pi_g}(s \vert s_g) > V^{\pi'}(s \vert s') \; \forall s, s' \neq s_g$, then $\pi^*(s) = \pi_g(s \vert s_g)$. This occurs for example if $-\infty < R_\cE(s, a) < 0$ and if the episode terminates when the agent reaches $s_g$. Then $V^{\pi_g}$ is a bounded negative number, and $V^{\pi'}$ for all other states is $-\infty$. We show that if every option is guaranteed to achieve its subgoal, then there must exist at least one sequence of options that satisfies the specification.

We then give the condition for the hierarchically optimal meta-policy $\mu^*(s)$ to always achieve the FSA goal state $f_g$. In our context, hierarchical optimality means that the meta-policy is optimal over the available options.
Let $\mu'(f, s \vert f')$ be the hierarchically optimal goal-conditioned meta-policy for reaching FSA state $f'$.
If the hierarchically optimal meta-policy equals the goal-conditioned meta-policy for reaching the FSA goal state $f_g$, i.e. $\mu^*(f, s) = \mu_g(f, s \vert f_g)$, then $\mu^*(f, s)$ will always reach $f_g$. In terms of value functions: let $V^{\mu'}(f, s \vert f')$ be the expected return for $\mu'$. If $V^{\mu_g}(f, s \vert f_g) > V^{\mu'}(f, s \vert f') \forall f, s, f' \neq f_g$, then $\mu^* = \mu_g$. This occurs if all FSA rewards $R_F(f) > 0$, all environment rewards $-\infty < R_\cE(s, a) < 0$, and the episode only terminates when the agent reaches $f_g$. Then $V^{\mu_g}$ is a bounded negative number, and $V^{\mu'}$ for all other states is $-\infty$. Because LOF uses the Bellman update equations to learn the meta-policy, the LOF meta-policy will converge to the hierarchically optimal meta-policy.

Consider the SMDP where planning is allowed over low-level actions, and let us call it the ``hierarchical MDP'' (HMDP) with optimal policy $\pi^*_{HMDP}$. 
Our result is:

\begin{theorem}\label{theorem}
Given that the conditions for satisfaction and hierarchical optimality are met, the LOF hierarchically optimal meta-policy $\mu_g$ with optimal option sub-policies $\pi_g$ has the same expected returns as the optimal policy $\pi^*_{HMDP}$ and satisfies the task specification.
\end{theorem}

\subsection{Composability}
 
The results in Sec.~\ref{sec:proofs} guarantee that LOF's learned model can be composed to satisfy new tasks. Furthermore, the composed policy has the same properties as the original policy -- satisfaction and optimality. LOF's possession of composability along with satisfaction and optimality derives from two facts: 1) Options are inherently composable because they can be executed in any order. 2) If the conditions of Thm.~\ref{theorem} are met, LOF is guaranteed to find a (hierarchically) optimal policy over the options that will satisfy any liveness property that uses subgoals associated with the options. The composability of LOF distinguishes it from other algorithms that can achieve satisfaction and optimality.

\section{Experiments \& Results}\label{sec:experiments}



\textbf{Experiments:} We performed experiments to demonstrate satisfaction and composability. For satisfaction, we measure cumulative reward over training steps. Cumulative reward is a proxy for satisfaction, as the environments can only achieve the maximum reward when they satisfy their tasks. For the composability experiments, we take the trained options and record how many meta-policy retraining steps it takes to learn an optimal meta-policy for a new task.

\textbf{Environments:} We measure the performance of LOF on three environments. The first environment is a discrete gridworld (Fig.~\ref{fig:discrete-domain}) called the ``delivery domain,'' as it can represent a delivery truck delivering packages to three locations ($a$, $b$, $c$) and having a home base $h$. There are also obstacles $o$ (the black squares). The second environment is called the reacher domain, from OpenAI Gym (Fig.~\ref{fig:continuous-domain}). It is a two-link arm that has continuous state and action spaces. There are four subgoals represented by colored balls: red $r$, green $g$, blue $b$, and yellow $y$. The third environment is called the pick-and-place domain, and it is a continuous 3D environment with a robotic Panda arm from CoppeliaSim and PyRep \cite{james2019pyrep}. It is inspired by the lunchbox-packing experiments of \citet{araki2019learning} in which subgoals $r$, $g$, and $b$ are food items that must be packed into lunchbox $y$. All environments also have an event proposition called $can$, which represents when the need to fulfill part of a task is cancelled.





\textbf{Tasks:} We test satisfaction and composability on four tasks. The first task is a ``sequential'' task. For the delivery domain, the LTL formula is $\Event(a \land \Event(b \land \Event(c \land \Event h))) \land \Always \neg o$ -- ``deliver package $a$, then $b$, then $c$, and then return to home $h$. And always avoid obstacles.'' The next task is the ``IF'' task (equivalent to the task shown in Fig.~\ref{fig:demo-ltl-fsa}): $(\Event (c \land \Event a) \land \Always \neg can) \lor (\Event c \land \Event can) \land \Always \neg o$ -- ``deliver package $c$, and then $a$, unless $a$ gets cancelled. And always avoid obstacles''. We call the third task the ``OR'' task, $\Event ((a \lor b) \land \Event c) \land \Always \neg o$  -- ``deliver package $a$ or $b$, then $c$, and always avoid obstacles''. The ``composite'' task has elements of all three of the previous tasks: $(\Event((a \lor b) \land \Event(c \land \Event h)) \land \Always \neg can) \lor (\Event((a \lor b) \land \Event h) \land \Event can) \land \Always \neg o$. ``Deliver package $a$ or $b$, and then $c$, unless $c$ gets cancelled, and then return to home $h$. And always avoid obstacles''. The tasks for the reacher and pick-and-place environments are equivalent, except that there are no obstacles for the reacher and arm to avoid.

The sequential task is meant to show that planning is efficient and effective even for long-time horizon tasks. The ``IF'' task shows that the agent's policy can respond to event propositions, such as being alerted that a delivery is cancelled. The ``OR'' task is meant to demonstrate the optimality of our algorithm versus a greedy algorithm, as discussed in Fig.~\ref{fig:lof-vs-greedy}. Lastly, the composite task shows that learning and planning are efficient and effective even for complex tasks. 

\textbf{Baselines:} We test four baselines against our algorithm. Our algorithm is \texttt{LOF-VI}, short for ``Logical Options Framework with Value Iteration,'' because it uses value iteration for high-level planning. \texttt{LOF-QL} uses Q-learning instead (details are in App.~\ref{sec:appendix-lof-ql}).
Unlike \texttt{LOF-VI}, \texttt{LOF-QL} does not need explicit knowledge of $T_F$, the FSA transition function.
\texttt{Greedy} is a naive implementation of task satisfaction; it uses its knowledge of the FSA to select the next subgoal with the lowest cost to attain.
This leaves it vulnerable to choosing suboptimal paths through the FSA, as shown in Fig.~\ref{fig:lof-vs-greedy}.
\texttt{Flat Options} uses the options framework with no knowledge of the FSA. Its SMDP formulation is not hierarchical -- the state space and transition function do not contain high-level states $\cF$ or transition function $T_F$.
The last baseline is \texttt{RM}, short for Reward Machines \citep{icarte2018using}.
Whereas LOF learn options to accomplish subgoals, \texttt{RM} learns sub-policies for every FSA state. App.~\ref{sec:appendix-lof-vs-rm} discusses the differences between \texttt{RM} and LOF in detail.

\begin{figure}
  \centering
  \includegraphics[width=0.75\columnwidth]{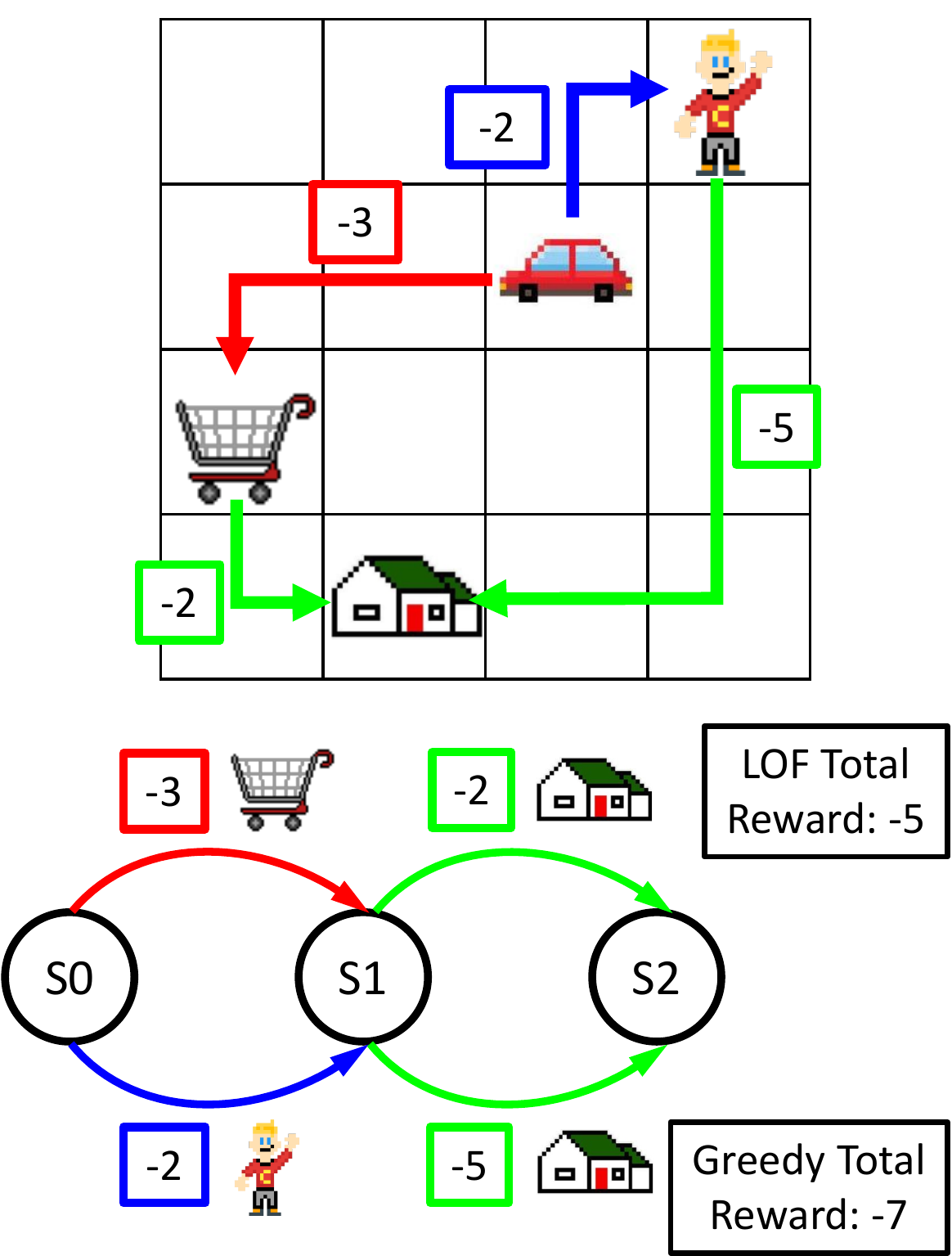}
  \caption{In this environment, the agent must either pick up the kid or go grocery shopping, and then go home (the ``OR'' task). Starting at $S0$, the greedy algorithm picks the next step in the FSA with the lowest cost (picking up the kid), which leads to a higher overall cost. LOF finds the optimal path through the FSA.}
  \label{fig:lof-vs-greedy}
\end{figure}

\begin{figure*}[!th]
\centering
\begin{subfigure}[b]{0.3\textwidth}
  \centering
  \includegraphics[width=3cm]{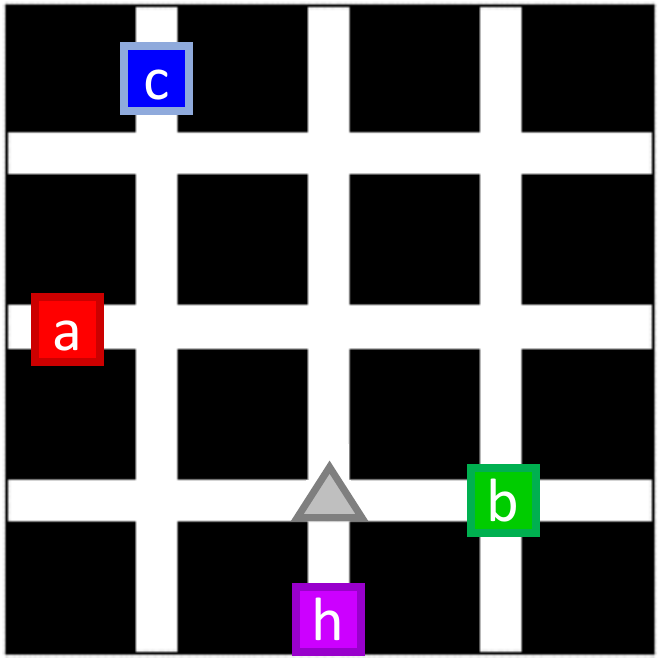}
  \caption{Delivery domain.}
  \label{fig:discrete-domain}
\end{subfigure} \hfill
\begin{subfigure}[b]{.3\textwidth}
  \centering
  \includegraphics[width=4cm]{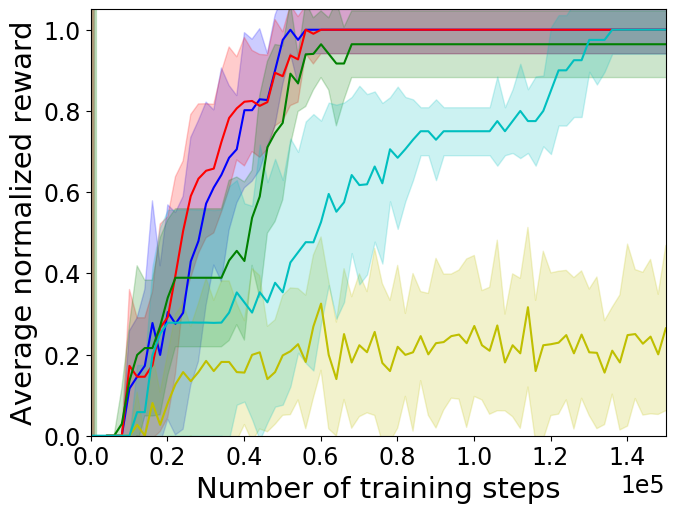}
  \caption{Satisfaction performance.}
  \label{fig:discrete-ave-perf}
 \end{subfigure} \hfill
\begin{subfigure}[b]{.3\textwidth}
  \centering
  \includegraphics[width=4cm]{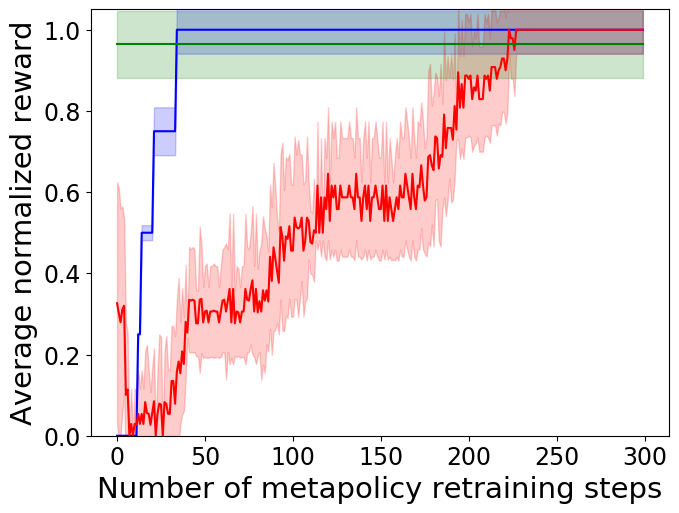}
  \caption{Composability performance.}
  \label{fig:discrete-composability}
\end{subfigure}

\begin{subfigure}[b]{.3\textwidth}
  \centering
  \includegraphics[width=4cm]{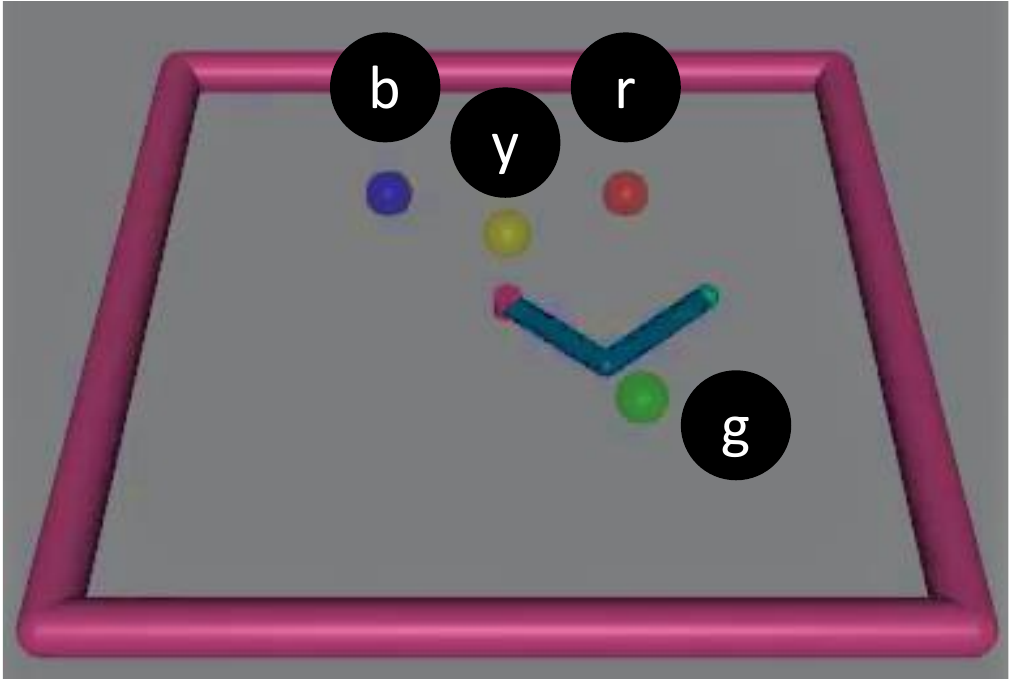}
  \caption{Reacher domain.}
  \label{fig:continuous-domain}
\end{subfigure} \hfill
\begin{subfigure}[b]{.3\textwidth}
  \centering
  \includegraphics[width=4cm]{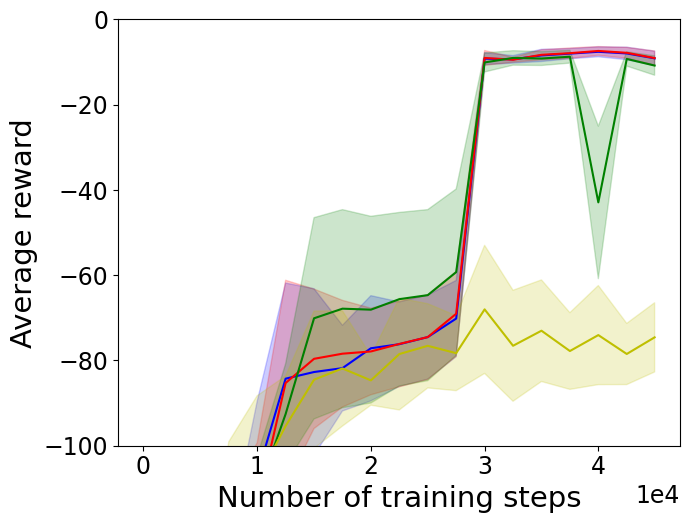}
  \caption{Satisfaction performance.}
  \label{fig:continuous-ave-perf}
 \end{subfigure} \hfill
 \begin{subfigure}[b]{.3\textwidth}
  \centering
  \includegraphics[width=4cm]{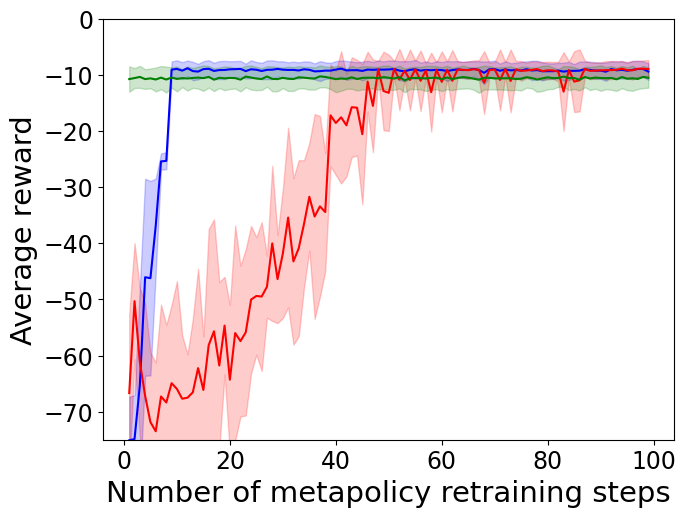}
  \caption{Composability performance.}
  \label{fig:continuous-composability}
 \end{subfigure}

\begin{subfigure}[b]{.3\textwidth}
  \centering
  \includegraphics[width=3.4cm]{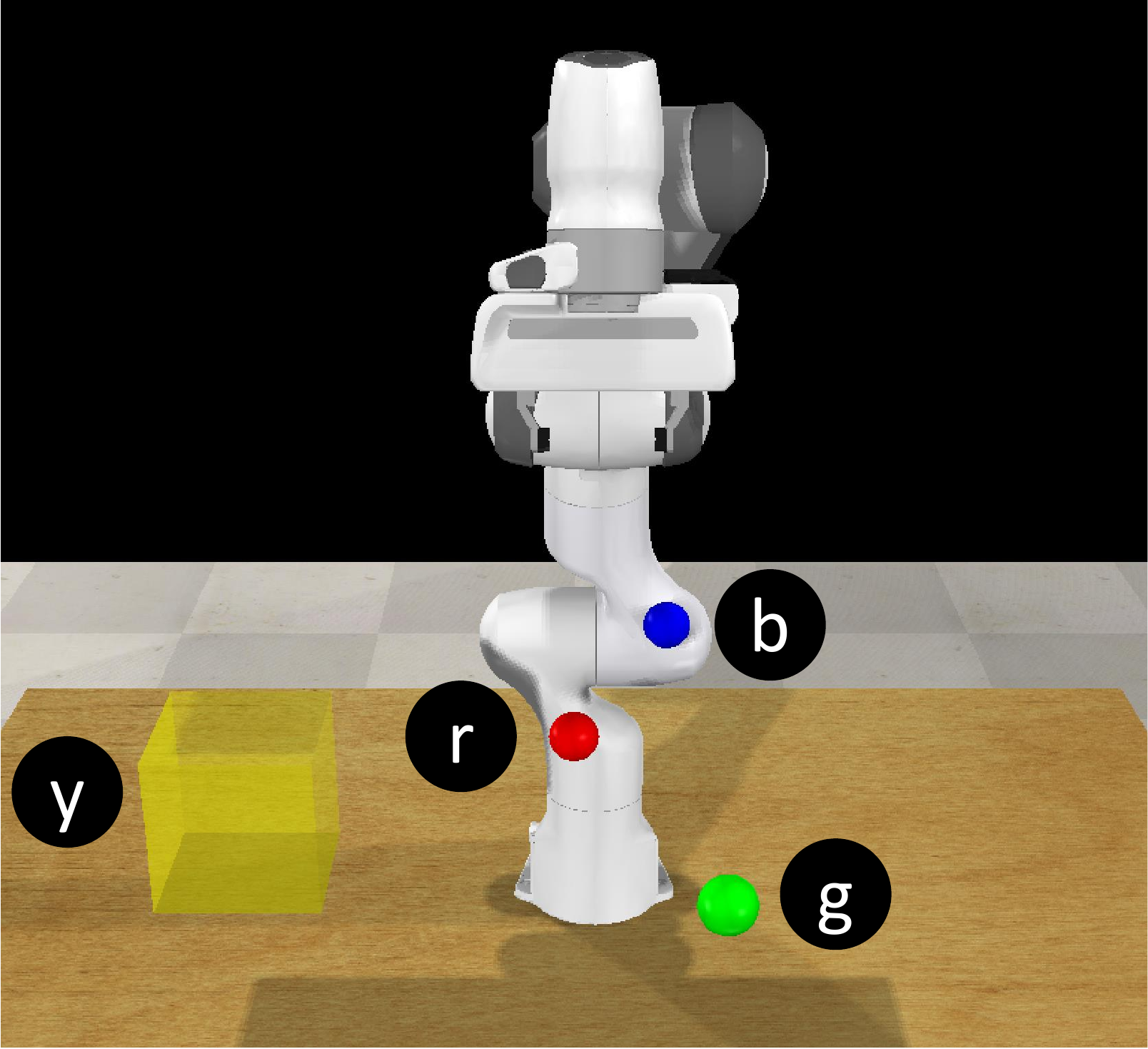}
  \caption{Pick-and-place domain.}
  \label{fig:arm-domain}
\end{subfigure} \hfill
\begin{subfigure}[b]{.3\textwidth}
  \centering
  \includegraphics[width=4cm]{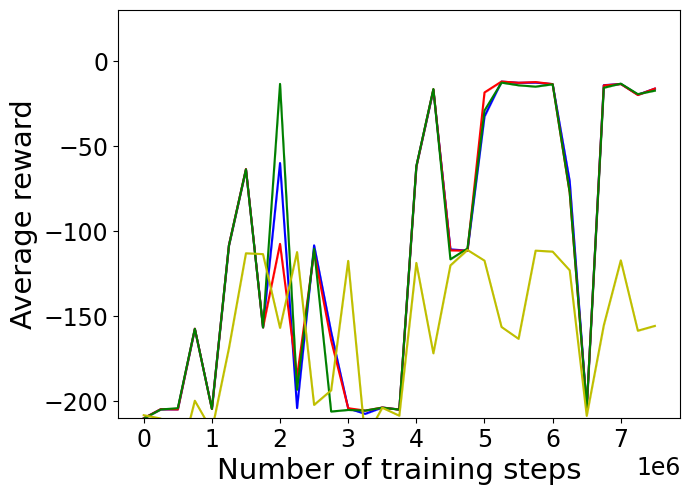}
  \caption{Satisfaction performance.}
  \label{fig:arm-ave-perf}
 \end{subfigure} \hfill
 \begin{subfigure}[b]{.3\textwidth}
  \centering
  \includegraphics[width=4cm]{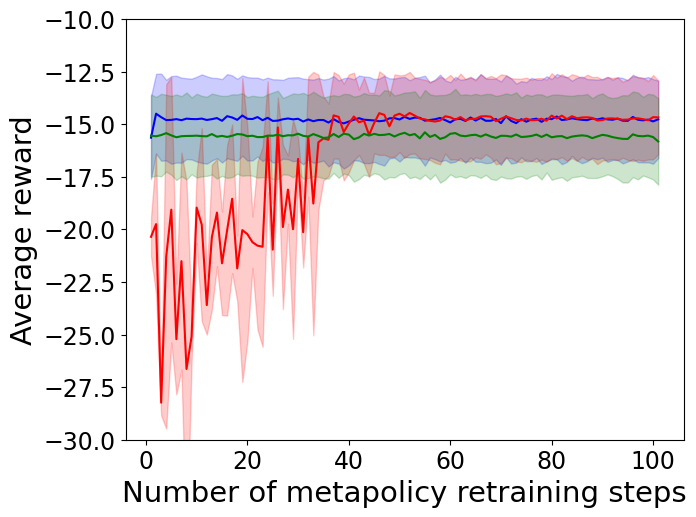}
  \caption{Composability performance.}
  \label{fig:arm-composability}
 \end{subfigure}

\begin{subfigure}[b]{\textwidth}
  \centering
  \includegraphics[height=0.6cm]{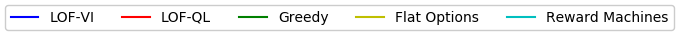}
  \label{fig:satisfaction-legend}
\end{subfigure}
\caption{Performance on the satisfaction and composability experiments, averaged over all tasks. Note that \texttt{LOF-VI} composes new meta-policies in just 10-50 retraining steps. The first row is the delivery domain, the second row is the reacher domain, and the third row is the pick-and-place domain.
All results, including \texttt{RM} performance on the reacher and pick-and-place domains, are in App.~\ref{sec:appendix-results}.}
\label{fig:satisfaction-experiments}
\end{figure*}

\textbf{Implementation:} For the delivery domain, options were learned using Q-learning with an $\epsilon$-greedy exploration policy.
\texttt{RM} was learned using the Q-learning for Reward Machines (QRM) algorithm described in \citep{icarte2018using}. For the reacher and pick-and-place domains, options were learned by using proximal policy optimization (PPO) \citep{schulman2017proximal} to train goal-oriented policy and value functions, which were represented using $128 \times 128$ and $128 \times 128 \times 128$ fully connected neural networks, respectively. Deep-QRM was used to train \texttt{RM}. The implementation details are discussed more fully in App.~\ref{sec:appendix-experiments}.

\subsection{Results}

\textbf{Satisfaction:} Results for the satisfaction experiments, averaged over all four tasks, are shown in Figs.~\ref{fig:discrete-ave-perf},~\ref{fig:continuous-ave-perf}, and~\ref{fig:arm-ave-perf}. (Results on all tasks are in App.~\ref{sec:appendix-results}). As expected, \texttt{Flat Options} shows no ability to satisfy tasks, as it has no knowledge of the FSAs. \texttt{Greedy} trains as quickly as \texttt{LOF-VI} and \texttt{LOF-QL}, but its returns plateau before the others because it chooses suboptimal paths in the composite and OR tasks. The difference is small in the continuous domains but still present. \texttt{LOF-QL} achieves as high a return as \texttt{LOF-VI}, but it is less composable (discussed below). \texttt{RM} learns much more slowly than the other methods. This is because for \texttt{RM}, a reward is only given for reaching the goal state, whereas in the LOF-based methods, options are rewarded for reaching their subgoals, so during training LOF-based methods have a richer reward function than \texttt{RM}. For the continuous domains, \texttt{RM} takes an order of magnitude more steps to train, so we left it out of the figures for clarity (see App. Figs.~\ref{fig:appendix-continuous-satifaction-experiments-with-rm} and \ref{fig:appendix-arm-satifaction-experiments-with-rm}). However, in the continuous domains, \texttt{RM} eventually achieves a higher return than the LOF-based methods. This is because for those domains, we define the subgoals to be spherical regions rather than single states, violating one of the conditions for optimality. Therefore, for example, it is possible that the meta-policy does not take advantage of the dynamics of the arm to swing through the subgoals more efficiently. \texttt{RM} does not have this condition and learns a single policy that can take advantage of inter-subgoal dynamics to learn a more optimal policy.

\textbf{Composability:} The composability experiments were done on the three composable baselines, \texttt{LOF-VI}, \texttt{LOF-QL}, and \texttt{Greedy}. App.~\ref{sec:appendix-lof-vs-rm} discusses why \texttt{RM} is not composable. \texttt{Flat Options} is not composable because its formulation does not include the FSA $\cT$. Therefore it is completely incapable of recognizing and adjusting to changes in the FSA. The composability results are shown in Figs.~\ref{fig:discrete-composability},~\ref{fig:continuous-composability}, and~\ref{fig:arm-composability}. \texttt{Greedy} requires no retraining steps to ``learn'' a meta-policy on a new FSA -- given its current FSA state, it simply chooses the next available FSA state that has the lowest cost to achieve. However, its meta-policy may be arbitrarily suboptimal. \texttt{LOF-QL} learns optimal (or in the continuous case, close-to-optimal) policies, but it takes $\sim$50-250 retraining steps, versus $\sim$10-50 for \texttt{LOF-VI}. Therefore \texttt{LOF-VI} strikes a balance between \texttt{Greedy} and \texttt{LOF-QL}, requiring far fewer steps than \texttt{LOF-QL} to retrain, and achieving better performance than \texttt{Greedy}.

\section{Related Work}


We distinguish our work from related work in HRL by its possession of three desirable properties -- composability, satisfaction, and optimality. Most other works possess two of these properties at the cost of the other.

\textbf{Not Composable:} The previous work most similar to ours is \citet{icarte2018using, icarte2019learning}, which introduces a method to solve tasks defined by automata called Reward Machines. Their method learns a sub-policy for every state of the automaton that achieves satisfaction and optimality. However, the learned sub-policies have limited composability because they end up learning a specific path through the automaton, and if the structure of the automaton is changed, there is no guarantee that the sub-policies will be able to satisfy the new automaton without re-training. By contrast, LOF learns a sub-policy for every subgoal, independent of the automaton, and therefore the sub-policies can be arranged to satisfy arbitrary tasks. Another similar work is Logical Value Iteration (LVI) \citep{araki2019learning, araki2020deep}. LVI defines a hierarchical MDP and value iteration equations that find satisfying and optimal policies; however, the algorithm is limited to discrete domains and has limited composability. A number of HRL algorithms use reward shaping to guide the agent through the states of an automaton \citep{li2017reinforcement, li2019formal, camacho2019ltl, hasanbeig2018logically, jothimurugan2019composable, shah2020planning, yuan2019modular}. While these algorithms can guarantee satisfaction and sometimes optimality, they cannot be composed because their policies are not hierarchical. Another approach is to use a symbolic planner to find a satisfying sequence of tasks and use an RL agent to learn and execute that sequence of tasks \citep{gordon2019should, illanes2020symbolic, lyu2019sdrl}. However, the meta-controllers of \citet{gordon2019should} and \citet{lyu2019sdrl} are not composable as they are trained together with the low-level controllers. Although the work of \citet{illanes2020symbolic} is amenable to transfer learning, it is not composable. \citet{paxton2017combining, mason2017assured} use logical constraints to guide exploration, and while these approaches are also satisfying and optimal, they are not composable as the agent is trained for a specific set of rules.
LOF is composable unlike the above methods because it has a hierarchical action space with high-level options. Once the options are learned, they can be composed arbitrarily.

\textbf{Not Satisfying:} Most hierarchical frameworks cannot satisfy tasks. Instead, they focus on using state and action abstractions to make learning more efficient \citep{dietterich2000hierarchical, dayan1993feudal, parr1998reinforcement, diuk2008object, oh2019planning}.
The options framework \citep{sutton1999between} stands out because of its composability and its guarantee of hierarchical optimality, which is why we based our work off of it. There is also a class of HRL algorithms
that builds on the idea of goal-oriented policies that can navigate to nearby subgoals \citep{eysenbach2019search, ghosh2018learning, faust2018prm}. By sampling sequences of subgoals and using a goal-oriented policy to navigate between them, these algorithms can travel much longer distances than a policy can travel on its own. Although these algorithms are ``composable'' in that they can navigate to far-away goals without further training, they are not able to solve tasks.
\citet{andreas2017modular} present an algorithm for solving simple policy ``sketches'' which is also composable; however, sketches are considerably less expressive than automata and linear temporal logic, which we use.
Unlike the above methods, LOF is satisfying because it has a hierarchical state space with low-level MDP states and high-level FSA states. Therefore LOF can satisfy tasks by learning policies that reach the FSA goal state.

\textbf{Not Optimal:} In HRL, there are at least three types of optimality -- hierarchical, recursive, and overall. As defined in \citet{dietterich2000hierarchical}, the hierarchically optimal policy is the optimal policy given the constraints of the hierarchy, and recursive optimality is when a policy is optimal given the policies of its children. For example, the options framework is hierarchically optimal, while MAXQ and abstract MDPs \citep{gopalan2017planning} are recursively optimal. The method described in \citet{kuo2020encoding} is fully composable, but not optimal as it uses a recurrent neural network to generate a sequence of high-level actions and is therefore not guaranteed to find optimal policies. LOF is hierarchically optimal because it finds an optimal meta-policy over the high-level options, and as we state in the paper, there are also conditions under which the overall policy is optimal.

\section{Discussion and Conclusion}

In this work, we claim that LOF has a unique combination of three properties: satisfaction, optimality, and composability. We state and prove the conditions for satisfaction and optimality in Sec.~\ref{sec:proofs}. The experimental results confirm our claims while also pointing out some weaknesses. \texttt{LOF-VI} achieves optimal or near-optimal policies and trains an order of magnitude faster than the existing work most similar to it, \texttt{RM}. However, the optimality condition that each subgoal be associated with one state cannot be met for continuous domains, and therefore \texttt{RM} eventually outperforms \texttt{LOF-VI}. But even when optimality is not guaranteed, \texttt{LOF-VI} is hierarchically optimal, which is why it outperforms \texttt{Greedy} in the composite and OR tasks.
Next, the composability experiments show that \texttt{LOF-VI} can compose its learned options to perform new tasks in about 10-50 iterations. Although \texttt{Greedy} requires no retraining steps, it is a tiny fraction of the tens of thousands of steps required to learn the original policy. 
Lastly, we have shown that LOF learns policies efficiently, and that it can be used with a variety of domains and policy-learning algorithms. 

\section*{Acknowledgments}

Toyota Research Institute provided funds to support this work. This material is also supported by the Office of Naval Research grant ONR N00014-18-1-2830, and the Under Secretary of Defense for Research and Engineering under Air Force Contract No. FA8702-15-D-0001. Any opinions, findings, conclusions or recommendations expressed in this material are those of the author(s) and do not necessarily reflect the views of the Under Secretary of Defense for Research and Engineering.

\bibliography{main}

\begin{thebibliography}{38}
\providecommand{\natexlab}[1]{#1}
\providecommand{\url}[1]{\texttt{#1}}
\expandafter\ifx\csname urlstyle\endcsname\relax
  \providecommand{\doi}[1]{doi: #1}\else
  \providecommand{\doi}{doi: \begingroup \urlstyle{rm}\Url}\fi

\bibitem[Abel \& Winder(2019)Abel and Winder]{abel2019expected}
Abel, D. and Winder, J.
\newblock The expected-length model of options.
\newblock In \emph{IJCAI}, 2019.

\bibitem[Alpern \& Schneider(1987)Alpern and Schneider]{alpern1987recognizing}
Alpern, B. and Schneider, F.~B.
\newblock Recognizing safety and liveness.
\newblock \emph{Distributed computing}, 2\penalty0 (3):\penalty0 117--126,
  1987.

\bibitem[Andreas et~al.(2017)Andreas, Klein, and Levine]{andreas2017modular}
Andreas, J., Klein, D., and Levine, S.
\newblock Modular multitask reinforcement learning with policy sketches.
\newblock In \emph{International Conference on Machine Learning}, pp.\
  166--175, 2017.

\bibitem[Araki et~al.(2019)Araki, Vodrahalli, Leech, Vasile, Donahue, and
  Rus]{araki2019learning}
Araki, B., Vodrahalli, K., Leech, T., Vasile, C.~I., Donahue, M., and Rus, D.
\newblock Learning to plan with logical automata.
\newblock In \emph{Proceedings of Robotics: Science and Systems},
  FreiburgimBreisgau, Germany, June 2019.
\newblock \doi{10.15607/RSS.2019.XV.064}.

\bibitem[Araki et~al.(2020)Araki, Vodrahalli, Leech, Vasile, Donahue, and
  Rus]{araki2020deep}
Araki, B., Vodrahalli, K., Leech, T., Vasile, C.~I., Donahue, M., and Rus, D.
\newblock Deep bayesian nonparametric learning of rules and plans from
  demonstrations with a learned automaton prior.
\newblock In \emph{AAAI}, pp.\  10026--10034, 2020.

\bibitem[Baier \& Katoen(2008)Baier and Katoen]{Baier08}
Baier, C. and Katoen, J.
\newblock \emph{{Principles of model checking}}.
\newblock MIT Press, 2008.
\newblock ISBN 978-0-262-02649-9.

\bibitem[Bhatia et~al.(2010)Bhatia, Kavraki, and Vardi]{bhatia2010sampling}
Bhatia, A., Kavraki, L.~E., and Vardi, M.~Y.
\newblock Sampling-based motion planning with temporal goals.
\newblock In \emph{2010 IEEE International Conference on Robotics and
  Automation}, pp.\  2689--2696. IEEE, 2010.

\bibitem[Camacho et~al.(2019)Camacho, Icarte, Klassen, Valenzano, and
  McIlraith]{camacho2019ltl}
Camacho, A., Icarte, R.~T., Klassen, T.~Q., Valenzano, R.~A., and McIlraith,
  S.~A.
\newblock Ltl and beyond: Formal languages for reward function specification in
  reinforcement learning.
\newblock In \emph{IJCAI}, volume~19, pp.\  6065--6073, 2019.

\bibitem[Clarke et~al.(2001)Clarke, Grumberg, and Peled]{Clark01}
Clarke, E.~M., Grumberg, O., and Peled, D.
\newblock \emph{Model Checking}.
\newblock MIT Press, 2001.
\newblock ISBN 978-0-262-03270-4.

\bibitem[Dayan \& Hinton(1993)Dayan and Hinton]{dayan1993feudal}
Dayan, P. and Hinton, G.~E.
\newblock Feudal reinforcement learning.
\newblock In \emph{Advances in neural information processing systems}, pp.\
  271--278, 1993.

\bibitem[Dietterich(2000)]{dietterich2000hierarchical}
Dietterich, T.~G.
\newblock Hierarchical reinforcement learning with the maxq value function
  decomposition.
\newblock \emph{Journal of artificial intelligence research}, 13:\penalty0
  227--303, 2000.

\bibitem[Diuk et~al.(2008)Diuk, Cohen, and Littman]{diuk2008object}
Diuk, C., Cohen, A., and Littman, M.~L.
\newblock An object-oriented representation for efficient reinforcement
  learning.
\newblock In \emph{Proceedings of the 25th international conference on Machine
  learning}, pp.\  240--247, 2008.

\bibitem[Duret-Lutz et~al.(2016)Duret-Lutz, Lewkowicz, Fauchille, Michaud,
  Renault, and Xu]{Duret16}
Duret-Lutz, A., Lewkowicz, A., Fauchille, A., Michaud, T., Renault, E., and Xu,
  L.
\newblock Spot 2.0 --- a framework for {LTL} and $\omega$-automata
  manipulation.
\newblock In \emph{Proceedings of the 14th International Symposium on Automated
  Technology for Verification and Analysis (ATVA'16)}, volume 9938 of
  \emph{Lecture Notes in Computer Science}, pp.\  122--129. Springer, October
  2016.
\newblock \doi{10.1007/978-3-319-46520-3_8}.

\bibitem[Eysenbach et~al.(2019)Eysenbach, Salakhutdinov, and
  Levine]{eysenbach2019search}
Eysenbach, B., Salakhutdinov, R.~R., and Levine, S.
\newblock Search on the replay buffer: Bridging planning and reinforcement
  learning.
\newblock In \emph{Advances in Neural Information Processing Systems}, pp.\
  15220--15231, 2019.

\bibitem[Faust et~al.(2018)Faust, Oslund, Ramirez, Francis, Tapia, Fiser, and
  Davidson]{faust2018prm}
Faust, A., Oslund, K., Ramirez, O., Francis, A., Tapia, L., Fiser, M., and
  Davidson, J.
\newblock Prm-rl: Long-range robotic navigation tasks by combining
  reinforcement learning and sampling-based planning.
\newblock In \emph{2018 IEEE International Conference on Robotics and
  Automation (ICRA)}, pp.\  5113--5120. IEEE, 2018.

\bibitem[Ghosh et~al.(2018)Ghosh, Gupta, and Levine]{ghosh2018learning}
Ghosh, D., Gupta, A., and Levine, S.
\newblock Learning actionable representations with goal-conditioned policies.
\newblock \emph{arXiv preprint arXiv:1811.07819}, 2018.

\bibitem[Gopalan et~al.(2017)Gopalan, Littman, MacGlashan, Squire, Tellex,
  Winder, Wong, et~al.]{gopalan2017planning}
Gopalan, N., Littman, M.~L., MacGlashan, J., Squire, S., Tellex, S., Winder,
  J., Wong, L.~L., et~al.
\newblock Planning with abstract markov decision processes.
\newblock In \emph{Twenty-Seventh International Conference on Automated
  Planning and Scheduling}, 2017.

\bibitem[Gordon et~al.(2019)Gordon, Fox, and Farhadi]{gordon2019should}
Gordon, D., Fox, D., and Farhadi, A.
\newblock What should i do now? marrying reinforcement learning and symbolic
  planning.
\newblock \emph{arXiv preprint arXiv:1901.01492}, 2019.

\bibitem[Hasanbeig et~al.(2018)Hasanbeig, Abate, and
  Kroening]{hasanbeig2018logically}
Hasanbeig, M., Abate, A., and Kroening, D.
\newblock Logically-constrained reinforcement learning.
\newblock \emph{arXiv preprint arXiv:1801.08099}, 2018.

\bibitem[Icarte et~al.(2018)Icarte, Klassen, Valenzano, and
  McIlraith]{icarte2018using}
Icarte, R.~T., Klassen, T., Valenzano, R., and McIlraith, S.
\newblock Using reward machines for high-level task specification and
  decomposition in reinforcement learning.
\newblock In \emph{International Conference on Machine Learning}, pp.\
  2107--2116, 2018.

\bibitem[Icarte et~al.(2019)Icarte, Waldie, Klassen, Valenzano, Castro, and
  McIlraith]{icarte2019learning}
Icarte, R.~T., Waldie, E., Klassen, T., Valenzano, R., Castro, M., and
  McIlraith, S.
\newblock Learning reward machines for partially observable reinforcement
  learning.
\newblock In \emph{Advances in Neural Information Processing Systems}, pp.\
  15523--15534, 2019.

\bibitem[Illanes et~al.(2020)Illanes, Yan, Icarte, and
  McIlraith]{illanes2020symbolic}
Illanes, L., Yan, X., Icarte, R.~T., and McIlraith, S.~A.
\newblock Symbolic plans as high-level instructions for reinforcement learning.
\newblock In \emph{Proceedings of the International Conference on Automated
  Planning and Scheduling}, volume~30, pp.\  540--550, 2020.

\bibitem[James et~al.(2019)James, Freese, and Davison]{james2019pyrep}
James, S., Freese, M., and Davison, A.~J.
\newblock Pyrep: Bringing v-rep to deep robot learning.
\newblock \emph{arXiv preprint arXiv:1906.11176}, 2019.

\bibitem[Jothimurugan et~al.(2019)Jothimurugan, Alur, and
  Bastani]{jothimurugan2019composable}
Jothimurugan, K., Alur, R., and Bastani, O.
\newblock A composable specification language for reinforcement learning tasks.
\newblock In \emph{Advances in Neural Information Processing Systems}, pp.\
  13041--13051, 2019.

\bibitem[Kansou(2019)]{kansou2019converting}
Kansou, B. K.~A.
\newblock Converting asubset of ltl formula to buchi automata.
\newblock \emph{International Journal of Software Engineering \& Applications
  (IJSEA)}, 10\penalty0 (2), 2019.

\bibitem[Kuo et~al.(2020)Kuo, Katz, and Barbu]{kuo2020encoding}
Kuo, Y.-L., Katz, B., and Barbu, A.
\newblock Encoding formulas as deep networks: Reinforcement learning for
  zero-shot execution of ltl formulas.
\newblock \emph{arXiv preprint arXiv:2006.01110}, 2020.

\bibitem[Li et~al.(2017)Li, Vasile, and Belta]{li2017reinforcement}
Li, X., Vasile, C.-I., and Belta, C.
\newblock Reinforcement learning with temporal logic rewards.
\newblock In \emph{2017 IEEE/RSJ International Conference on Intelligent Robots
  and Systems (IROS)}, pp.\  3834--3839. IEEE, 2017.

\bibitem[Li et~al.(2019)Li, Serlin, Yang, and Belta]{li2019formal}
Li, X., Serlin, Z., Yang, G., and Belta, C.
\newblock A formal methods approach to interpretable reinforcement learning for
  robotic planning.
\newblock \emph{Science Robotics}, 4\penalty0 (37), 2019.

\bibitem[Lyu et~al.(2019)Lyu, Yang, Liu, and Gustafson]{lyu2019sdrl}
Lyu, D., Yang, F., Liu, B., and Gustafson, S.
\newblock Sdrl: interpretable and data-efficient deep reinforcement learning
  leveraging symbolic planning.
\newblock In \emph{Proceedings of the AAAI Conference on Artificial
  Intelligence}, volume~33, pp.\  2970--2977, 2019.

\bibitem[Mason et~al.(2017)Mason, Calinescu, Kudenko, and
  Banks]{mason2017assured}
Mason, G.~R., Calinescu, R.~C., Kudenko, D., and Banks, A.
\newblock Assured reinforcement learning with formally verified abstract
  policies.
\newblock In \emph{9th International Conference on Agents and Artificial
  Intelligence (ICAART)}. York, 2017.

\bibitem[Oh et~al.(2019)Oh, Patel, Nguyen, Huang, Pavlick, and
  Tellex]{oh2019planning}
Oh, Y., Patel, R., Nguyen, T., Huang, B., Pavlick, E., and Tellex, S.
\newblock Planning with state abstractions for non-markovian task
  specifications.
\newblock \emph{arXiv preprint arXiv:1905.12096}, 2019.

\bibitem[Parr \& Russell(1998)Parr and Russell]{parr1998reinforcement}
Parr, R. and Russell, S.~J.
\newblock Reinforcement learning with hierarchies of machines.
\newblock In \emph{Advances in neural information processing systems}, pp.\
  1043--1049, 1998.

\bibitem[Paxton et~al.(2017)Paxton, Raman, Hager, and
  Kobilarov]{paxton2017combining}
Paxton, C., Raman, V., Hager, G.~D., and Kobilarov, M.
\newblock Combining neural networks and tree search for task and motion
  planning in challenging environments.
\newblock In \emph{2017 IEEE/RSJ International Conference on Intelligent Robots
  and Systems (IROS)}, pp.\  6059--6066. IEEE, 2017.

\bibitem[Schulman et~al.(2017)Schulman, Wolski, Dhariwal, Radford, and
  Klimov]{schulman2017proximal}
Schulman, J., Wolski, F., Dhariwal, P., Radford, A., and Klimov, O.
\newblock Proximal policy optimization algorithms.
\newblock \emph{arXiv preprint arXiv:1707.06347}, 2017.

\bibitem[Shah et~al.(2020)Shah, Li, and Shah]{shah2020planning}
Shah, A., Li, S., and Shah, J.
\newblock Planning with uncertain specifications (puns).
\newblock \emph{IEEE Robotics and Automation Letters}, 5\penalty0 (2):\penalty0
  3414--3421, 2020.

\bibitem[Sutton et~al.(1999)Sutton, Precup, and Singh]{sutton1999between}
Sutton, R.~S., Precup, D., and Singh, S.
\newblock Between mdps and semi-mdps: A framework for temporal abstraction in
  reinforcement learning.
\newblock \emph{Artificial intelligence}, 112\penalty0 (1-2):\penalty0
  181--211, 1999.

\bibitem[Yuan et~al.(2019)Yuan, Hasanbeig, Abate, and
  Kroening]{yuan2019modular}
Yuan, L.~Z., Hasanbeig, M., Abate, A., and Kroening, D.
\newblock Modular deep reinforcement learning with temporal logic
  specifications.
\newblock \emph{arXiv preprint arXiv:1909.11591}, 2019.

\bibitem[Zhang \& Sridharan(2020)Zhang and Sridharan]{zhang2020survey}
Zhang, S. and Sridharan, M.
\newblock A survey of knowledge-based sequential decision making under
  uncertainty.
\newblock \emph{arXiv preprint arXiv:2008.08548}, 2020.

\end{thebibliography}
\bibliographystyle{icml2021}

\clearpage

\appendix

\section{Formulation of Logical Options Framework with Safety Automaton}\label{sec:appendix-lof}

In this section, we present a more general formulation of LOF than that presented in the paper. In the paper, we make two assumptions that simplify the formulation. The first assumption is that the LTL specification can be divided into two independent formulae, a liveness property and a safety property: $\phi = \phi_{liveness} \land \phi_{safety}$. However, not all LTL formulae can be factored in this way. We show how LOF can be applied to LTL formulae that break this assumption. The second assumption is that the safety property takes a simple form that can be represented as a penalty on safety propositions. We show how LOF can be used with arbitrary safety properties.

\subsection{Automata and Propositions}

All LTL formulae can be translated into B\"uchi automata using automatic translation tools such as SPOT \citep{Duret16}. All B\"uchi automata can be decomposed into liveness and safety properties \citep{alpern1987recognizing}, so that automaton $\cW = \cW_{liveness} \times \cW_{safety}$. This is a generalization of the assumption that all LTL formulae can be divided into liveness and safety properties $\phi_{liveness}$ and $\phi_{safety}$. The liveness property $\cW_{liveness}$ must be an FSA, although this assumption could also be loosened to allow it to be a deterministic B\"uchi automaton via some minor modifications (allowing multiple goal states to exist and continuing episodes indefinitely, even once a goal state has been reached).

As in the main text, we assume that there are three types of propositions -- subgoals $\cP_G$, safety propositions $\cP_S$, and event propositions $\cP_E$. The event propositions have set values and can occur in both $\cW_{liveness}$ and $\cW_{safety}$. Safety propositions only appear in $\cW_{safety}$. Subgoal propositions only appear in $\cW_{liveness}$. Each subgoal may only be associated with one state. Note that after writing a specification and decomposing it into $\cW_{liveness}$ and $\cW_{safety}$, it is possible that some subgoals may unexpectedly appear in $\cW_{safety}$. This can be dealt with by creating ``safety twins'' of each subgoal -- safety propositions that are associated with the same low-level states as the subgoals and can therefore substitute for them in $\cW_{safety}$.

Subgoals are propositions that the agent must achieve in order to reach the goal state of $\cW_{liveness}$. Although event propositions can also define transitions in $\cW_{liveness}$, we assume that ``achieving'' them is not necessary in order to reach the goal state. In other words, we assume that from any state in $\cW_{liveness}$, there is a path to the goal state that involves only subgoals. This is because in our formulation, the event propositions are meant to serve as propositions that the agent has no control over, such as receiving a phone call. If satisfaction of the liveness property were to depend on such a proposition, then it would be impossible to guarantee satisfaction. However, if the user is unconcerned with guaranteeing satisfaction, then specifying a liveness property in which satisfaction depends on event propositions is compatible with LOF.

Safety propositions may only occur in $\cW_{safety}$ and are associated with things that the agent ``must avoid''. This is because every state of $\cW_{safety}$ is an accepting state \citep{alpern1987recognizing}, so all transitions between the states are non-violating. However, any undefined transition is not allowed and is a violation of the safety property. In our formulation, we assign costs to violations, so that violations are allowed but come at a cost. In practice, it also may be the case that the agent is in a low-level state from which it is impossible to reach the goal state without violating the safety property. In our formulation, satisfaction of the liveness property (but not the safety property) is still guaranteed in this case, as the finite cost associated with violating the rule is less than the infinite cost of not satisfying the liveness property, so the optimal policy for the agent will be to violate the rule in order to satisfy the task (see the proofs, Appendix~\ref{sec:appendix-proofs}). This scenario can be avoided in several ways. For example, do not specify an environment in which it is only possible for the agent to satisfy the task by violating a rule. Or, instead of prioritizing satisfaction of the task, it is possible to instead prioritize satisfaction of the safety property. In this case, satisfaction of the liveness property would not be guaranteed but satisfaction of the safety property would be guaranteed. This could be accomplished by terminating the rollout if a safety violation occurs.

We assume that event propositions are observed -- in other words, that we know the values of the event propositions from the start of a rollout. This is because we are planning in a fully observable setting, so we must make this assumption to guarantee convergence to an optimal policy. However, the partially observable case is much more interesting, in which the values of the event propositions are not known until the agent checks or the environment randomly reveals their values. This case is beyond the scope of this paper; however, LOF can still guarantee satisfaction and composability in this setting, just not optimality.

Proposition labeling functions relate states to propositions: $T_{P_G} : \cS \rightarrow 2^{\cP_G}$, $T_{P_S} : \cS \rightarrow 2^{\cP_S}$, and $T_{P_E} : 2^{\cP_E} \rightarrow \{0, 1\}$.

Given these definitions of propositions, it is possible to define the liveness and safety properties formally. $\cW_{liveness} = (\cF, \cP_G \cup \cP_E, T_F, R_F, f_0, f_g)$. $\cF$ is the set of states of the liveness property. The propositions can be either subgoals $\cP_G$ or event propositions $\cP_E$. The transition function relates the current FSA state and active propositions to the next FSA state, $T_F : \cF \times 2^{\cP_G} \times 2^{\cP_E} \times \cF \rightarrow [0, 1]$. The reward function assigns a reward to the current FSA state, $R_F : \cF \rightarrow \mathbb{R}$. We assume there is one initial state $f_0$ and one goal state $f_g$.

The safety property is a B\"uchi automaton $\cW_{safety} = (\cF_S, \cP_S \cup \cP_E, T_S, R_S, F_0)$. $\cF_S$ are the states of the automaton. The propositions can be safety propositions $\cP_S$ or event propositions $\cP_E$. The transition function $T_S$ relates the current state and active propositions to the next state, $T_S : \cF_S \times 2^{\cP_S} \times 2^{\cP_E} \times \cF_S \rightarrow [0, 1]$. The reward function relates the automaton state and safety propositions to rewards (or costs), $R_S : \cF_S \times 2^{\cP_S} \rightarrow \mathbb{R}$. $F_0$ defines the set of initial states. We do not specify an accepting condition because for safety properties, every state is an accepting state.

\subsection{The Environment MDP}

There is a low-level environment MDP $\cE = (\cS, \cA, R_E, T_E, \gamma)$. $\cS$ is the state space and $\cA$ is the action space. They can be either discrete or continuous. $R_E$ is the low-level reward function that characterizes, for example, time, distance, or actuation costs. $T_E : \cS \times \cA \times \cS \rightarrow [0, 1]$ is the transition function and $\gamma$ is the discount factor. Unlike in the simpler formulation in the paper, we do not combine $R_E$ and the safety automaton reward function $R_S$ in the MDP formulation $\cE$.

\begin{algorithm}[!t]
\caption{Learning and Planning with Logical Options}\label{alg:lof-full}
\begin{algorithmic}[1]
    \STATE \textbf{Given:} \par
        Propositions $\cP$ partitioned into subgoals $\cP_G$, safety propositions $\cP_S$, and event propositions $\cP_E$ \par
        $\cW_{liveness} = (\cF, \cP_G \cup \cP_E, T_F, R_F, f_0, f_g)$ \par
        $\cW_{safety} = (\cF_S, \cP_S \cup \cP_E, T_S, R_S, F_0)$ \par
        Low-level MDP $\cE = (\cS, \cA, R_E, T_E, \gamma)$ \par
        Proposition labeling functions $T_{P_G} : \cS \rightarrow 2^{\cP_G}$,  $T_{P_S} : \cS \rightarrow 2^{\cP_S}$, and $T_{P_E} : 2^{\cP_E} \rightarrow \{0, 1\}$ \par
    \STATE \textbf{To learn:} 
        \STATE Set of options $\cO$, one for each subgoal proposition $p \in \cP_G$
        \STATE Meta-policy $\mu(f, f_s, s, o)$ along with $Q(f, f_s, s, o)$ and $V(f, f_s, s)$
    \STATE \textbf{Learn logical options:}\label{alg:logical-options-appendix}
    \STATE For every $p$ in $\cP_G$, learn an option for achieving $p$, $o_p = (\cI_{o_p}, \pi_{o_p}, \beta_{o_p}, R_{o_p}, T_{o_p})$ \par
        \STATE $ \cI_{o_p} = \cS $
        \STATE $ \beta_{o_p} = \begin{cases} 1 &\quad\text{if } p \in T_{P_G}(s)\\
                                             0 &\quad\text{otherwise} \\
                                             \end{cases} $
        \STATE $\pi_{o_p} = $ optimal policy on $\cE \times \cW_{safety}$ with rollouts terminating when $p \in T_{P_G}(s)$
        \STATE $T_{o_p}(f'_s, s' \vert f_s, s) = \begin{cases} \sum\limits_{k=1}^{\infty} p(f'_s, k) \gamma^k &\quad\text{if } p \in T_P(s') \\
                                             0 &\quad\text{otherwise} \\
                                             \end{cases}$
        \STATE $R_{o_p}(f_s, s) = \mathbb{E} [ \cR_\cE(f_{s}, s, a_1) + \gamma \cR_\cE(f_{s, 1}, s_1, a_2) + \dots + \gamma^{k-1} \cR_\cE(f_{s, k-1}, s_{k-1}, a_{k}) ]$ 
    \STATE \textbf{Find a meta-policy $\mu$ over the options:}
        \STATE Initialize $Q : \cF \times \cF_S \times \cS \times \cO \rightarrow \mathbb{R}$ and $V : \cF \times \cF_S \times \cS \rightarrow \mathbb{R}$ to $0$
        \FOR{$(k, f, f_s, s) \in [1, \dots, n] \times \cF \times \cF_S \times \cS$:}
        \FOR{$o \in \cO$:}
        \STATE $Q_k(f, f_s, s, o) \leftarrow R_F(f)R_o(f_s, s) + $ \par $\sum\limits_{f' \in \cF} \sum\limits_{f'_s \in \cF_S} \sum\limits_{\bar{p}_e \in 2^{\cP_E}} \sum\limits_{s' \in \cS} T_F(f' \vert f, T_{P_G}(s'), \bar{p}_e)$ \par $ T_S(f'_s \vert f_s, T_{P_S}(s'), \bar{p}_e) T_{P_E}(\bar{p}_e)$ \par $T_o(s' \vert s) V_{k-1}(f', f'_s, s')$
        \ENDFOR
        \STATE $V_k(f, f_s, s) \leftarrow \max\limits_{o \in \cO} Q_k(f, f_s, s, o)$
        \ENDFOR
    \STATE $\mu(f, f_s, s, o) = \argmax\limits_{o \in \cO} Q(f, f_s, s, o)$
    \STATE \textbf{Return: } Options $\cO$, meta-policy $\mu(f, f_s, s, o)$, and $Q(f, f_s, s, o), V(f, f_s, s)$
\end{algorithmic}
\end{algorithm}

\subsection{Logical Options}

We associate every subgoal $p_g$ with an option $o_{p_g} = (\cI_{p_g}, \pi_{p_g}, \beta_{p_g}, R_{p_g}, T_{p_g})$. Every $o_{p_g}$ has a policy $\pi_{p_g}$ whose goal is to reach the state $s_{p_g}$ where $p_g$ is true. Option policies are learned by training on the product of the environment and the safety automaton, $\cE \times \cW_{safety}$ and terminating training only when $s_{p_g}$ is reached. $R_\cE : \cF_S \times \cS \times \cA \rightarrow \mathbb{R}$ is the reward function of the product MDP $\cE \times \cW_{safety}$. There are many reward-shaping policy-learning algorithms that specify how to define $R_\cE$. In fact, learning a policy for $\cE \times \cW_{safety}$ is the sort of hierarchical learning problem that many reward-shaping algorithms excel at, including Reward Machines \citep{icarte2018using} and \citep{li2017reinforcement}. This is because in LOF, safety properties are not composable, so using a learning algorithm that is satisfying and optimal but not composable to learn the safety property is appropriate. Alternatively, there are many scenarios where $\cW_{safety}$ is a trivial automaton in which each safety proposition is associated with its own state, as we describe in the main paper, so penalties can be assigned to propositions and the state of the agent in $\cW_{safety}$ can be ignored.

Note that since the options are trained independently, one limitation of our formulation is that the safety properties cannot depend on the liveness state. In other words, when an agent reaches a new subgoal, the safety property cannot change. However, the workaround for this is not too complicated. First, if the liveness state affects the safety property, this implies that liveness propositions such as subgoals may be in the safety property. In this case, as we described above, the subgoals present in the safety property need to be substituted with ``safety twin'' propositions. Then during option training, a policy-learning algorithm must be chosen that will learn sub-policies for all of the safety property states, even if those states are only reached after completing a complicated task (for example, all of the sub-policies could be trained in parallel as in \citep{icarte2018using}). Lastly, during meta-policy learning and during rollouts, when a new option is chosen, the current state of the safety property must be passed to the new option.

The components of the logical options are defined starting at Alg.~\ref{alg:lof-full} line~\ref{alg:logical-options-appendix}. Note that for stochastic low-level transitions, the number of time steps $k$ at which the option terminates is stochastic and characterized by a distribution function. In general this distribution function must be learned, which is a challenging problem. However, there are many approaches to solving this problem; \citep{abel2019expected} contains an excellent discussion.

The most notable difference between the general formulation and the formulation in the paper is that the option policy, transition, and reward functions are functions of the safety automaton state $f_s$ as well as the low-level state $s$. This makes Logical Value Iteration more complicated, because in the paper, we could assume we knew the final state of each option (i.e., the state of its associated subgoal $s_g$). But now, although we still assume that the option will terminate at $s_g$, we do not know which safety automaton state it will terminate in, so the transition model must learn a distribution over safety automaton states, and Logical Value Iteration must account for this uncertainty.


\subsection{Hierarchical SMDP}

Given a low-level environment $\cE$, a liveness property $\cW_{liveness}$, a safety property $\cW_{safety}$, and logical options $\cO$,  we can define a hierarchical semi-Markov Decision Process (SMDP) $\cM = \cE \times \cW_{liveness} \times \cW_{safety}$ with options $\cO$ and reward function $R_{SMDP}$. This SMDP differs significantly from the SMDP in the paper in that the safety property $\cW_{safety}$ is now an integral part of the formulation. $R_{SMDP}(f, f_s, s, o) = R_F(f)R_o(f_s, o)$.

\subsection{Logical Value Iteration}

A value function and Q-function are found for the SMDP using the Bellman update equations:

\begin{align}
\begin{split}\label{eq:q-update-appendix}
Q_k(&f, f_s, s, o) \leftarrow R_F(f)R_o(f_s, s) + \sum_{f' \in \cF} \sum_{f'_s \in \cF_S} \\
&\sum_{\bar{p}_e \in 2^{\cP_E}} \sum_{s' \in \cS} T_F(f' \vert f, T_{P_G}(s'), \bar{p}_e) \\
&T_S(f'_s \vert f_s, T_{P_S}(s'), \bar{p}_e) T_{P_E}(\bar{p}_e) T_o(s' \vert s) V_{k-1}(f', f'_s, s')
\end{split}\\
V_k(&f, f_s, s) \leftarrow \max_{o \in \cO} Q_k(f, f_s, s, o) \label{eq:v-update-appendix}
\end{align}

\section{Proofs and Conditions for Satisfaction and Optimality}\label{sec:appendix-proofs}

The proofs are based on the more general LOF formulation of Appendix~\ref{sec:appendix-lof}, as results on the more general formulation also apply to the simpler formulation used in the paper.

\begin{definition}
\label{definition:options}
Let the reward function of the environment be $R_\cE(f_s, s, a)$, which is some combination of $R_E(s, a)$ and $R_S(f_s, \bar{p}_s) = R_S(f_s, T_{P_S}(s))$. Let $\pi' : \cF_S \times \cS \times \cA \times \cS \rightarrow [0, 1]$ be the optimal goal-conditioned policy for reaching a state $s'$. In the case of a goal-conditioned policy, the reward function is $R_\cE$, and the objective is to maximize the expected reward with the constraint that $s'$ is reached in a finite amount of time. We assume that every state $s'$ is reachable from any state $s$, a standard regularity assumption in MDP literature. Let $V^{\pi'}(f_s, s \vert s')$ be the optimal expected cumulative reward for reaching $s'$ from $s$ with goal-conditioned policy $\pi'$. Let $s_g$ be the state associated with the subgoal, and let $\pi_g$ be the optimal goal-conditioned policy associated with reaching $s_g$. Let $\pi^*$ be the optimal policy for the environment $\cE$.
\end{definition}

\begin{condition}
The optimal policy for the option must be the same as the goal-conditioned policy that has subgoal $s_g$ as its goal: $\pi^*(f_s, s) = \pi_g(f_s, s \vert s_g)$. In other words,  $V^{\pi_g}(f_s, s \vert s_g) > V^{\pi'}(f_s, s \vert s') \;\; \forall f_s, s, s' \neq s_g$.
\label{lemma:subgoals}
\end{condition}
This condition guarantees that the optimal option policy will always reach the subgoal $s_g$.
It can be achieved by setting all rewards $-\infty < R_\cE(f_s, s, a) < 0$ and terminating the episode only when the agent reaches $s_g$. Therefore the expected return for reaching $s_g$ is a bounded negative number, and the expected return for all other states is $-\infty$. 



\begin{lemma}
Given that the goal state of $\cW_{liveness}$ is reachable from any other state using only subgoals and that there is an option for every subgoal and that all the options meet Condition~\ref{lemma:subgoals}, there exists a meta-policy that can reach the FSA goal state from any non-trap state in the FSA. 
\label{lemma:satisfaction-exists}
\end{lemma}
\begin{proof}
This follows from the fact that transitions in $\cW_{liveness}$ are determined by achieving subgoals, and it is given that there exists an option for achieving every subgoal. Therefore, it is possible for the agent to execute any sequence of subgoals, and at least one of those sequences must satisfy the task specification since the FSA representing the task specification is finite and satisfiable, and the goal state $f_g$ is reachable from every FSA state $f \in \cF$ using only subgoals.
\end{proof}


\begin{definition}
From \citet{dietterich2000hierarchical}: A \textbf{hierarchically optimal} policy for an MDP or SMDP is a policy that achieves the highest cumulative reward among all policies consistent with the given hierarchy.
\label{def:hierarchical-optimality}
\end{definition}

In our case, this means that the hierarchically optimal meta-policy is optimal over the available options.

\begin{definition}
Let the expected cumulative reward function of an option $o$ started at state $(f_s, s)$ be $R_o(f_s, s)$. Let the reward function on the SMDP be $R_{SMDP}(f, f_s, s, o) = R_F(f)R_o(f_s, s)$ with $R_F(f) \geq 0$\footnote{The assumption that $R_{SMDP}(f, f_s, s, o) = R_F(f)R_o(f_s, s)$  and $R_{HMDP}(f, f_s, s, a) = R_F(f)R_\cE(f_s, s, a)$ can be relaxed so that $R_{SMDP}$ and $R_{HMDP}$ are functions that are monotonic increasing in the low-level rewards $R_o$ and $R_\cE$, respectively.}. Let $\mu' : \cF \times \cF_S \times \cS \times \cO \times \cF \rightarrow [0, 1] $ be the hierarchically optimal goal-conditioned meta-policy for achieving liveness state $f'$. The objective of the meta-policy is to maximize the reward function $R_{SMDP}$ with the constraint that it reaches $f'$ in a finite number of time steps. Let $V^{\mu'}(f, f_s, s \vert f')$ be the hierarchically optimal return for reaching $f'$ from $(f, f_s, s)$ with goal-conditioned meta-policy $mu'$. Let $\mu^*$ be the hierarchically optimal policy for the SMDP. Let $f_g$ be the goal state, and $\mu_g$ be the hierarchically optimal goal-conditioned meta-policy for achieving the goal state.
\end{definition}

\begin{condition}
The hierarchically optimal meta-policy must be the same as the goal-conditioned meta-policy that has the FSA goal state $f_g$ as its goal: $\mu^*(f,f_s,s) = \mu_g(f,f_s,s \vert f_g)$. In other words, $V^{\mu_g}(f,f_s, s \vert f_g) > V^{\mu'}(f,f_s, s \vert f') \;\; \forall f, f_s, s, f' \neq f_g$.
\label{lemma:satisfaction}
\end{condition}

This condition guarantees that the hierarchically optimal meta-policy will always go to the FSA goal state $f_g$ (thereby satisfying the specification). Here is an example of how this condition can be achieved: If $-\infty < R_\cE(f_s, s, a) < 0 \;\; \forall s$, then $R_o(f_s, s) < 0 \;\; \forall f_s, o, s$. Then if $R_F(f) > 0$ (in our experiments, we set $R_F(f) = 1 \;\; \forall f$), $R_{SMDP}(f, f_s, s, o) = R_F(f)R_o(f_s, s) < 0$, and if the episode only terminates when the agent reaches the goal state, then the expected return for reaching $f_g$ is a bounded negative number, and the expected return for all other states is $-\infty$.



\begin{lemma}
From \citep{sutton1999between}: Value iteration on an SMDP converges to the hierarchically optimal policy.
\label{lemma:converge-hierarchical}
\end{lemma}

Therefore, the meta-policy found using the Logical Options Framework converges to a hierarchically optimal meta-policy that satisfies the task specification as long as Conditions~\ref{lemma:subgoals} and~\ref{lemma:satisfaction} are met.


\begin{definition}
Consider the SMDP where planning is allowed over the low-level actions instead of the options. We will call this the hierarchical MDP (HMDP), as this MDP is the product of the low-level environment $\cE$, the liveness property $\cW_{liveness}$, and the safety property $\cW_{safety}$. Let $R_F(f) > 0 \; \forall f$, and let $R_{HMDP}(f, f_s, s, a) = R_F(f)R_\cE(f_s, s, a)$, and let $\pi^*_{HMDP}$ be the optimal policy for the HMDP.
\end{definition}

\begin{theorem}
Given Conditions~\ref{lemma:subgoals} and~\ref{lemma:satisfaction}, the hierarchically optimal meta-policy $\mu_g$ with optimal option policies $\pi_g$ has the same expected returns as the HMDP optimal policy $\pi^*$ and satisfies the task specification.
\label{theorem:converge-optimal}
\end{theorem}
\begin{proof}
By Condition~\ref{lemma:subgoals}, every subgoal has an option associated with it whose optimal policy is to go to the subgoal. By Condition~\ref{lemma:satisfaction}, the hierarchically optimal meta-policy will reach the FSA goal state $f_g$. The meta-policy can only accomplish this by going to the subgoals in a sequence that satisfies the task specification. It does this by executing a sequence of options that correspond to a satisfying sequence of subgoals and are optimal in expectation. Therefore, since $R_F(f) > 0 \; \forall f$ and $R_{SMDP}(f, f_s, s, o) = R_F(f)R_o(f_s, s)$, and since the event propositions that affect the order of subgoals necessary to satisfy the task are independent random variables, the expected cumulative reward is a positive linear combination of the expected option rewards, and since all option rewards are optimal with respect to the environment and the meta-policy is optimal over the options, our algorithm attains the optimal expected cumulative reward.
\end{proof}

\section{Experimental Implementation}\label{sec:appendix-experiments}

We discuss the implementation details of the experiments in this section. Because the setups of the domains are analogous, we discuss the delivery domain first in every section and then briefly relate how the same formulation applies to the reacher and pick-and-place domains as well. In this section, we use the simpler formulation of the main paper and not the more general formulation discussed in Appendix~\ref{sec:appendix-lof}.

\subsection{Propositions}

The delivery domain has 7 propositions plus 4 composite propositions. The subgoal propositions are $\cP_G = \{a, b, c, h\}$. Each of these propositions is associated with a single state in the environment (see Fig.~\ref{fig:appendix-discrete-domain}). The safety propositions are $\cP_S = \{o, e\}$. $o$ is the obstacle proposition. It is associated with many states -- the black squares in Fig.~\ref{fig:appendix-discrete-domain}. $e$ is the empty proposition, associated with all of the white squares in the domain. This is the default proposition for when there are no other active propositions. The event proposition is $\cP_E = \{can\}$. $can$ is the ``cancelled'' proposition, representing when one of the subgoals has been cancelled.

To simplify the FSAs and the implementation, we make an assumption that multiple propositions cannot be true at the same state. However, it is reasonable for $can$ to be true at the subgoals, and therefore we introduce 4 composite propositions, $ca = a \land can$, $cb = b \land can$, $cc = c \land can$, $ch = h \land can$. These can be counted as event propositions without affecting the operation of the algorithm.

The reacher domain has analogous propositions. The subgoals are $r, g, b, y$ and correspond to $a, b, c, h$. The environment does not contain obstacles $o$ but does have safety proposition $e$, and it also has the event proposition $can$ and the composite propositions $cr, cg, cb, cy$ for when $can$ is true at the same time that a subgoal proposition is true. Another difference is that the subgoal propositions are associated with a small spherical region instead of a single state as in the delivery domain; this is a necessity for continuous domains and unfortunately breaks one of our conditions for optimality because the subgoals are now associated with multiple states instead of a single state. However, the LOF meta-policy will still converge to a hierarchically optimal policy.

The pick-and-place domain has subgoals $r, g, b, y$ like the reacher domain, and event proposition $can$. Like the reacher domain, the pick-and-place domain's subgoals become true in a region around the goal state, breaking one of the necessary conditions for optimality. However, the LOF meta-policy still converges to a hierarchically optimal policy.

\subsection{Reward Functions}

Next, we define the reward functions of the physical environment $R_E$, safety propositions $R_S$, and FSA states $R_F$. We realize that often in reinforcement learning, the algorithm designer has no control over the reward functions of the environment. However, in our case, there are no publicly available environments such as OpenAI Gym or the DeepMind Control Suite that we know of that have a high-level FSA built-in. Therefore, anyone implementing our algorithm will likely have to implement their own high-level FSA and define the rewards associated with it.

For the delivery domain, the low-level environment reward function $R_E : \cS \times \cA \rightarrow \mathbb{R}$ is defined to be $-1 \; \forall s, a$. In other words, it is a time/distance cost.

We assign costs to the safety propositions by defining the reward function $R_S : \cP_S \rightarrow \mathbb{R}$. All of the costs are $0$ except for the obstacle cost, $R_S(o) = -1000$. Therefore, there is a very high penalty for encountering an obstacle.

We define the environment reward function $R_\cE : \cS \times \cA \rightarrow \mathbb{R}$ to be $R_\cE(s, a) = R_E(s, a) + R_S(T_P(s))$. In other words, it is the sum of $R_E$ and $R_S$. This reward function meets Condition~\ref{lemma:subgoals} for the optimal option policies to always converge to their subgoals.

Lastly, we define $R_F : \cF \rightarrow \mathbb{R}$ to be $R_F(f) = 1 \; \forall f$. Therefore the SMDP cost $R_SMDP(f, s, o) = R_o(s)$ and meets Condition~\ref{lemma:satisfaction} so that the LOF meta-policy converges to the optimal policy.

The reacher environment has analogous reward functions. The safety reward function $R_S(p) = 0 \; \forall p \in \cP_S$ because there is no obstacle proposition. Also, the physical environment reward function differs during option training and meta-policy learning. For meta-policy learning, the reward function is $R_E(s, a) = -a^\top a - 0.1$ -- a time cost and an actuation cost. During option training, we speed learning by adding the distance to the goal state as a cost, instead of a time cost: $R_E(s, a) = -a^\top a - \vert\vert s - s_g \vert \vert^2$. Although the reward functions and value functions are different, the costs are analogous and lead to good performance as seen in the results. Note that this method can't be used for Reward Machines, because it trains sub-policies for FSA states, and the subgoals for FSA states are not known ahead of time, so distance to subgoal cannot be calculated.

The pick-and-place domain has reward functions analogous to the reacher domain's.

\subsection{Algorithm for \texttt{LOF-QL}}\label{sec:appendix-lof-ql}

The \texttt{LOF-QL} baseline uses Q-learning to learn the meta-policy instead of value iteration. We therefore use ``Logical Q-Learning'' equations in place of the Logical Value Iteration equations described in Eqs.~\ref{eq:q-update} and~\ref{eq:v-update} in the main text. The algorithm is described in Alg.~\ref{alg:lof-ql}. A benefit of using Q-learning instead of value iteration is that the transition function $T_F$ of the FSA $\cT$ does not have to be explicitly known, as the algorithm samples from the transitions rather than using $T_F$ explicitly in the formula. However, as described in the main text, this comes at the expense of reduced composability, as \texttt{LOF-QL} takes around $5x$ more iterations to converge to a new meta-policy than \texttt{LOF-VI} does. Let $Q_0(f, s, o)$ be initialized to be all $0$s. The Q update formulas are given in Alg.~\ref{alg:lof-ql} lines~\ref{eq:lof-ql-update1} and~\ref{eq:lof-ql-update2}.

\begin{algorithm}[!ht]
\caption{LOF with $\epsilon$-greedy Q-learning}\label{alg:lof-ql}
\begin{algorithmic}[1]
    \STATE \textbf{Given:} \par
        Propositions $\cP$ partitioned into subgoals $\cP_G$, safety propositions $\cP_S$, and event propositions $\cP_E$ \par
        Environment MDP $\cE = (\cS, \cA, T_E, R_\cE, \gamma)$ \par
        Logical options $\cO$ with reward models $R_o(s)$ and transition models $T_o(s' \vert s)$ \par
        Liveness property $\cT = (\cF, \cP_G \cup \cP_E, T_F, R_F, f_0, f_g)$ ($T_F$ does not have to be explicitly known if it can be sampled from a simulator) \par
        Learning rate $\alpha$, exploration probability $\epsilon$ \par
        Number of training episodes $n$, episode length $m$
    \STATE \textbf{To learn:} 
        \STATE Meta-policy $\mu(f, s, o)$ along with $Q(f, s, o)$ and $V(f, s)$
    \STATE \textbf{Find a meta-policy $\mu$ over the options:}
        \STATE Initialize $Q : \cF \times \cS \times \cO \rightarrow \mathbb{R}$ and $V : \cF \times \cS \rightarrow \mathbb{R}$ to $0$
        \FOR{$k \in [1, \dots, n] $:}
        \STATE Initialize FSA state $f \leftarrow 0$, $s$ a random initial state from $\cE$
        \STATE Draw $\bar{p}_e \sim T_{P_E}()$
        \FOR{$j \in [1, \dots, m]$:}
        \STATE With probability $\epsilon$ let $o$ be a random option; otherwise, $o \leftarrow \argmax\limits_{o' \in \cO} Q(f, s, o')$
        \STATE $s' \sim T_o(s)$
        \STATE $f' \sim T_F(T_{P_G}(s'), \bar{p}_e, f)$
        \STATE $Q_k(f, s, o) \leftarrow Q_{k-1}(f, s, o) + \alpha \big( R_F(f)R_o(s) + \gamma V(f', s') - Q_{k-1}(f, s, o) \big)$ \label{eq:lof-ql-update1}
        \STATE $V_k(f, s) \leftarrow \max\limits_{o' \in \cO} Q_k(f, s, o')$ \label{eq:lof-ql-update2}
        \STATE $f \leftarrow f'$
        \ENDFOR
        \ENDFOR
    \STATE $\mu(f, s, o) = \argmax\limits_{o \in \cO} Q(f, s, o)$
    \STATE \textbf{Return: } Options $\cO$, meta-policy $\mu(f, s, o)$ and Q- and value functions $Q(f, s, o), V(f, s)$
\end{algorithmic}
\end{algorithm}

\subsection{Comparison of LOF and Reward Machines}\label{sec:appendix-lof-vs-rm}

Figs.~\ref{fig:lof-vs-rm-1},~\ref{fig:lof-vs-rm-rm},~\ref{fig:lof-vs-rm-3}, and~\ref{fig:lof-vs-rm-4} give a visual overview of how LOF and Reward Machines work, and illustrate how they differ.

\begin{figure*}[!th]
\centering
\begin{subfigure}[t]{.35\textwidth}
  \centering
  \includegraphics[width=0.6\textwidth]{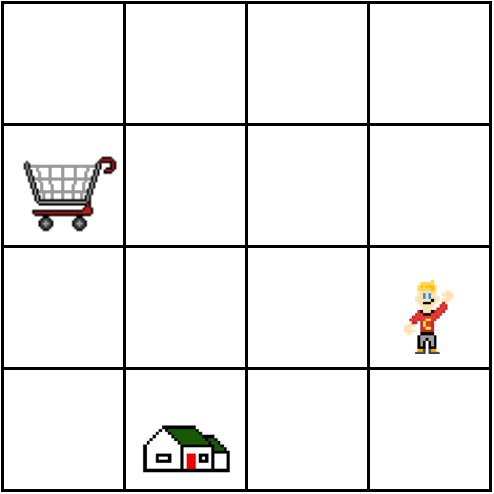}
  \caption{Environment MDP $\cE$.}
  \label{fig:lof-vs-rm-environment}
\end{subfigure} \hfill
\begin{subfigure}[t]{.6\textwidth}
  \centering
  \includegraphics[width=\textwidth]{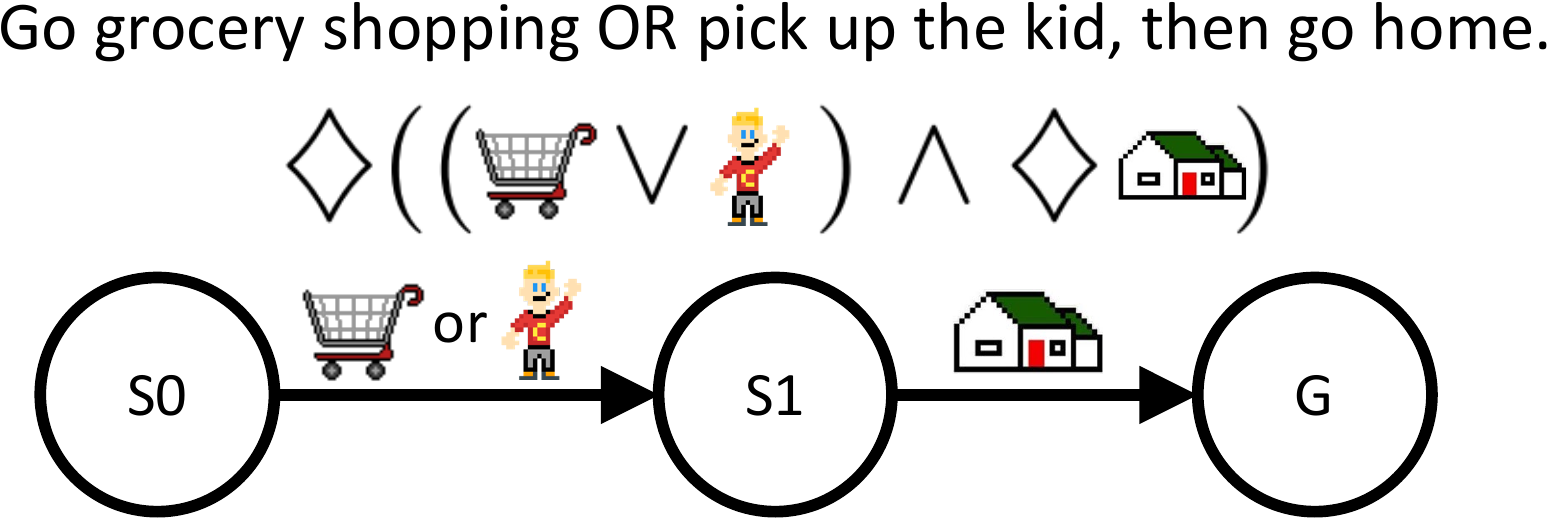}
  \caption{Liveness property $\cT$. The natural language rule can be represented as an LTL formula which can be translated into an FSA.}
  \label{fig:lof-vs-rm-fsa}
 \end{subfigure}
\caption{LOF and \texttt{RM} both require an environment MDP $\cE$ and an automaton $\cT$ that specifies a task.}
\label{fig:lof-vs-rm-1}
\end{figure*}

\begin{figure*}[!th]
  \centering
  \includegraphics[width=0.8\textwidth]{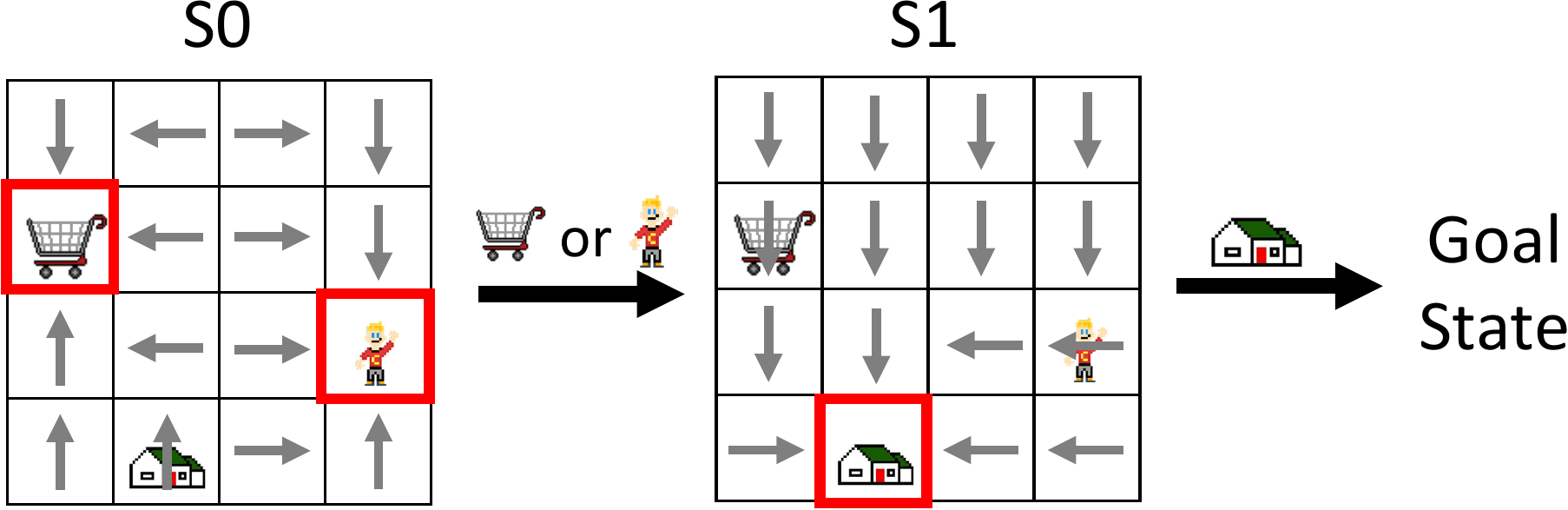}
  \caption{In \texttt{RM}, sub-policies are learned for each state of the automaton. In this case, in state $S0$, a sub-policy is learned that goes either to the shopping cart of the kid, whichever is closer. In state $S1$, the sub-policy goes to the house.}
  \label{fig:lof-vs-rm-rm}
\end{figure*}

\begin{figure*}[!th]
\centering
\begin{subfigure}[t]{\textwidth}
  \centering
  \includegraphics[width=0.85\textwidth]{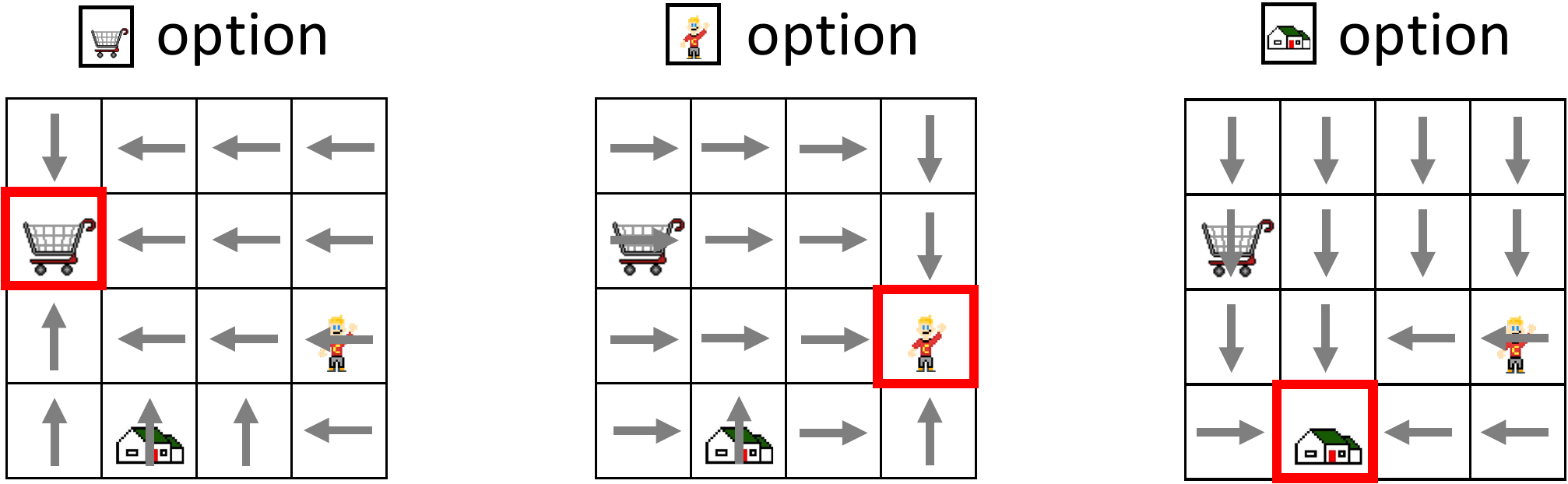}
  \caption{Step 1 of LOF: Learn a logical option for each subgoal.}
  \label{fig:lof-vs-rm-lof1}
\end{subfigure}
~
\begin{subfigure}[t]{\textwidth}
  \centering
  \includegraphics[width=0.85\textwidth]{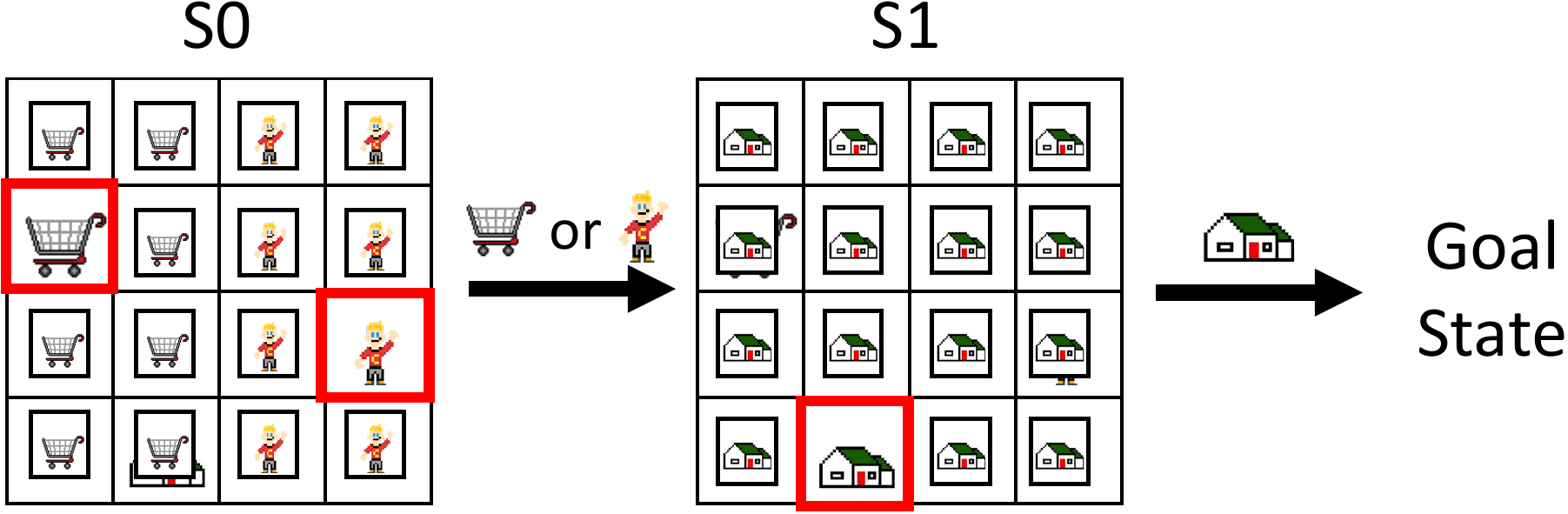}
  \caption{Step 2 of LOF: Use Logical Value Iteration to find a meta-policy that satisfies the liveness property. In this image, the boxed subgoals indicate that the corresponding option is the optimal option to take from that low-level state. The policy ends up being the same as \texttt{RM}'s policy -- in state $S0$, the optimal meta-policy chooses the ``grocery shoppping'' option if the grocery cart is closer and the ``pick up kid'' option if the kid is closer. In the state $S1$, the optimal meta-policy is to always choose the ``home'' option.}
  \label{fig:lof-vs-rm-lof2}
 \end{subfigure}
\caption{LOF has two steps. In (a) the first step, logical options are learned for each subgoal. In (b) the second step, a meta-policy is found using Logical Value Iteration.}
\label{fig:lof-vs-rm-3}
\end{figure*}

\begin{figure*}[!th]
 \begin{subfigure}[t]{\textwidth}
  \centering
  \includegraphics[width=0.5\textwidth]{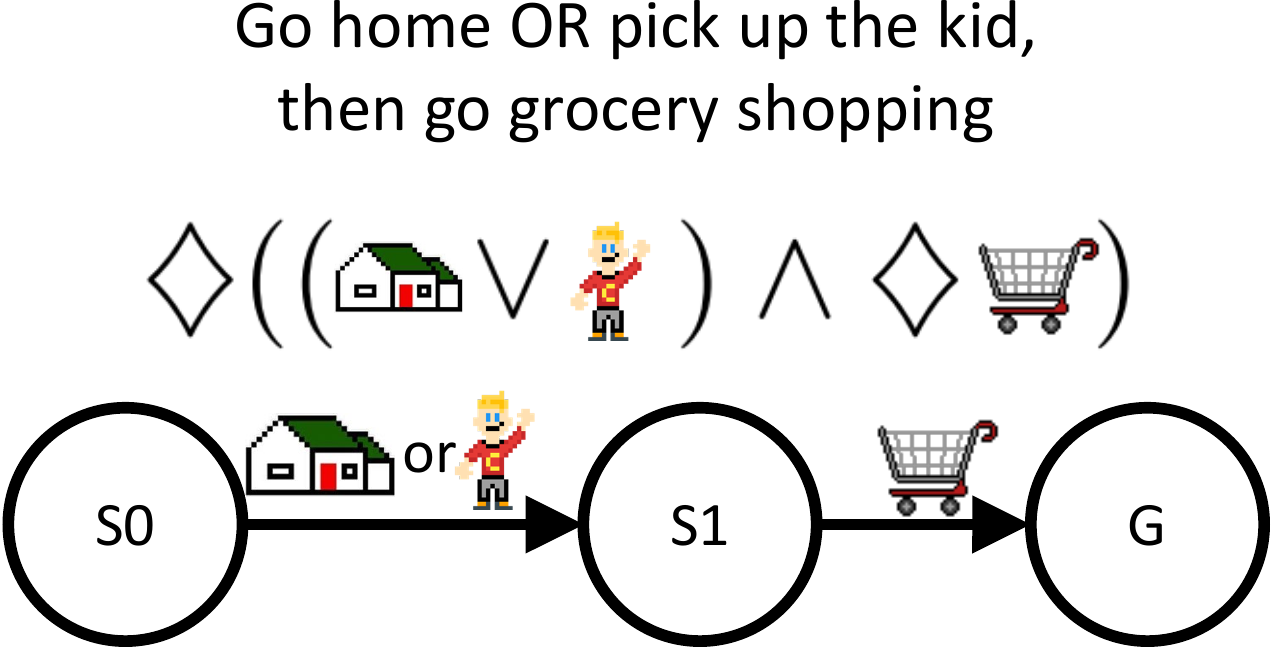}
  \caption{LOF can easily solve this new liveness property without training new options.}
  \label{fig:lof-vs-rm-new-task}
\end{subfigure}
\begin{subfigure}[t]{\textwidth}
  \centering
  \includegraphics[width=0.85\textwidth]{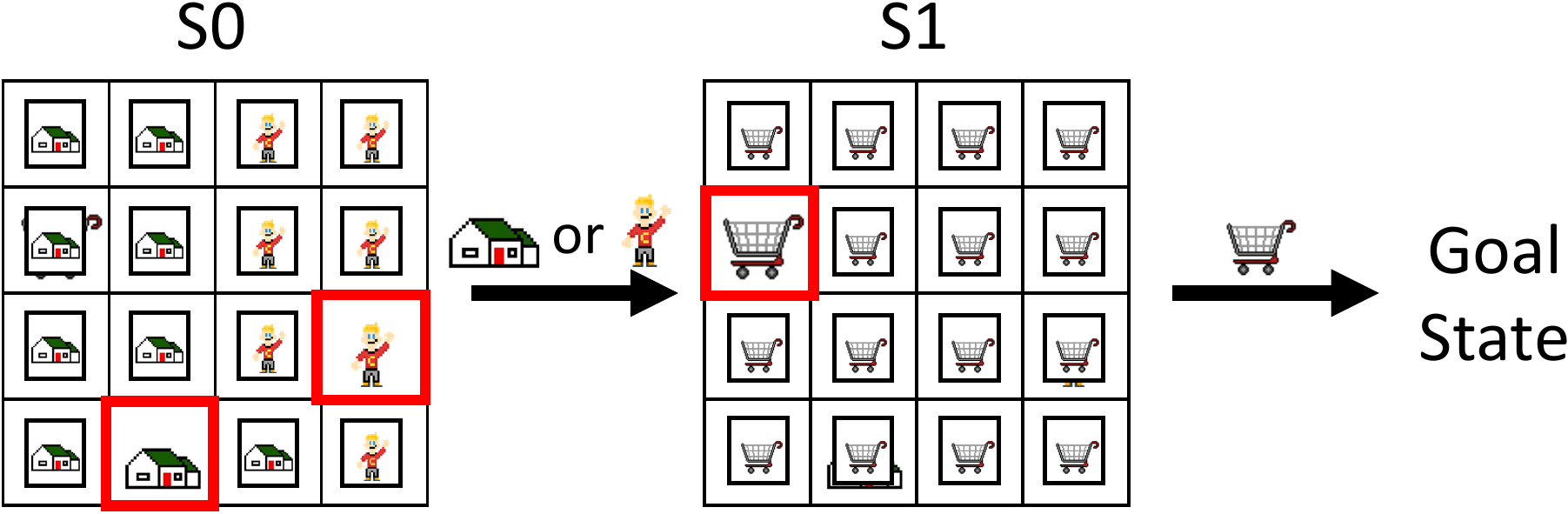}
  \caption{Logical Value Iteration can be used to find a meta-policy on the new task without the need to retrain the logical options. A new meta-policy can be found in 10-50 iterations. The new policy finds that in state $S0$, ``home'' option is optimal if the agent is closer to ``home'', and the ``kid'' option is optimal if the agent is closer to ``kid''. In state $S1$, the ``grocery shopping'' option is optimal everywhere.}
  \label{fig:lof-vs-rm-compose}
 \end{subfigure}
\caption{What distinguishes LOF from \texttt{RM} is that the logical options of LOF can be easily composed to solve new tasks. In this example, the new task is to go home or pick up the kid, then go grocery shopping. Logical Value Iteration can find a new meta-policy in 10-50 iterations without needing to relearn the options.}
\label{fig:lof-vs-rm-4}
\end{figure*}

\subsection{Tasks}\label{appendix:tasks}

We test the environments on four tasks, a ``sequential'' task (Fig.~\ref{fig:fsa-sequential}), an ``IF'' task (Fig.~\ref{fig:fsa-if}), an ``OR'' task (Fig.~\ref{fig:fsa-or}), and a ``composite'' task (Fig.~\ref{fig:fsa-composite}). The reacher domain has the same tasks, expect $r, g, b, y$ replace $a, b, c, h$, and there are no obstacles $o$. Note that in the LTL formulae, $\Always !o$ is the safety property $\phi_{safety}$; the preceding part of the formula is the liveness property $\phi_{liveness}$ used to construct the FSA.

\begin{figure}[!htb]
    \centering
    \resizebox{.95\linewidth}{!}{\begin{tikzpicture}[->,>=stealth',shorten >=1pt,auto,node distance=2.5cm, semithick,
initial text={}]
\tikzstyle{every state}=[text=black]

\node[state,initial] [initial where=left] (s0) { $init$ };
\node[state] [right of=s0] (s1) { $s_{ 1 }$ };
\node[state] [right of=s1] (s2) { $s_2$ };
\node[state] [right of=s2] (s3) { $s_3$ };
\node[state,accepting] [right of=s3] (s4) { $ goal $ };

\path[->]
(s0) edge [font=\scriptsize, right, above] node { $ a $ } (s1)
(s1) edge [font=\scriptsize, right, above] node { $ b $ } (s2)
(s2) edge [font=\scriptsize, right, above] node { $ c $ } (s3)
(s3) edge [font=\scriptsize, right, above] node { $ h $ } (s4)
;
\end{tikzpicture}}
    \caption{FSA for the sequential task. The LTL formula is $\Event(a \land \Event(b \land \Event(c \land \Event h))) \land \Always ! o$. The natural language interpretation is ``Deliver package $a$, then $b$, then $c$, and then return home $h$. And always avoid obstacles $o$''.}
    \label{fig:fsa-sequential}
\end{figure}
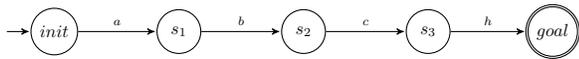

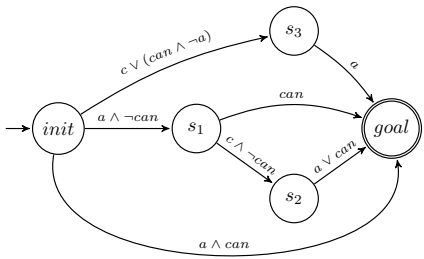
\begin{figure}[!htb]
    \centering
    \resizebox{.7\linewidth}{!}{\begin{tikzpicture}[->,>=stealth',shorten >=1pt,auto,node distance=2.5cm, semithick,
initial text={}]
\tikzstyle{every state}=[text=black]

\node[state,initial] [initial where=left] (s0) { $init$ };
\node[state] [right of=s0] (s1) { $s_{ 1 }$ };
\node[state] [below right of=s1, yshift=0.5cm] (s2) { $s_{ 2 }$ };
\node[state] [above right of=s1] (s3) { $s_{ 3 }$ };
\node[state,accepting] [below right of=s3] (s4) { $goal$ };

\path[->]
(s0) edge [font=\scriptsize,above] node { $ a \land \neg can $ } (s1)
(s0) edge [font=\scriptsize,sloped,above,bend left=10] node { $c \lor (can \land \neg a)$ } (s3)
(s0) edge [font=\scriptsize,sloped,above,bend right=100] node { $a \land can$  } (s4)

(s1) edge [font=\scriptsize,sloped,above] node { $c \land \neg can$ } (s2)
(s1) edge [font=\scriptsize,sloped,above, bend left=20] node { $can$ } (s4)

(s2) edge [font=\scriptsize,sloped,above] node { $a \lor can$ } (s4)

(s3) edge [font=\scriptsize,sloped,above, bend left=10] node { $ a $ } (s4)
;
\end{tikzpicture}}
    \caption{FSA for the IF task. The LTL formula is $(\Event (c \land \Event a) \land \Always ! can) \lor (\Event a \land \Event can) \land \Always ! o$. The natural language interpretation is ``Deliver package $c$, and then $a$, unless $a$ gets cancelled. And always avoid obstacles $o$''.}
    \label{fig:fsa-if}
\end{figure}

\begin{figure}[!htb]
    \centering
    \resizebox{.6\linewidth}{!}{\begin{tikzpicture}[->,>=stealth',shorten >=1pt,auto,node distance=2.5cm, semithick,
initial text={}]
\tikzstyle{every state}=[text=black]

\node[state,initial] [initial where=left] (s0) { $init$ };
\node[state] [right of=s0] (s1) { $s_{ 1 }$ };
\node[state,accepting] [right of=s1] (s2) { $ goal $ };

\path[->]
(s0) edge [font=\scriptsize, right, above] node { $ a \lor b$ } (s1)
(s1) edge [font=\scriptsize, right, above] node { $ c $ } (s2)
;
\end{tikzpicture}}
    \caption{FSA for the OR task. The LTL formula is $\Event ((a \lor b) \land \Event c) \land \Always ! o$. The natural language interpretation is ``Deliver package $a$ or $b$, then $c$, and always avoid obstacles $o$''.}
    \label{fig:fsa-or}
\end{figure}

\begin{figure}[!htb]
    \centering
    \resizebox{.95\linewidth}{!}{\begin{tikzpicture}[->,>=stealth',shorten >=1pt,auto,node distance=2.5cm, semithick,
initial text={}]
\tikzstyle{every state}=[text=black]

\node[state,initial] [initial where=left] (s0) { $init$ };
\node[state] [above right of=s0] (s1) { $s_{ 1 }$ };
\node[state] [right of=s1] (s2) { $s_{ 2 }$ };
\node[state] [right of=s2] (s3) { $s_{ 3 }$ };
\node[state] [below right of=s0, xshift=1cm] (s4) { $s_{ 4 }$ };
\node[state] [right of=s4] (s5) { $s_{ 5 }$ };
\node[state,accepting] [below right of=s3] (s6) { $goal$ };

\path[->]
(s0) edge [font=\scriptsize,sloped,above, bend left=10] node { $ (a \lor b) \land \neg can $ } (s1)
(s0) edge [font=\scriptsize,sloped,below,bend right=20] node { $can \land \neg a \land \neg b$ } (s4)
(s0) edge [font=\scriptsize,sloped,below, bend left=10] node { $ (a \lor b) \land can $ } (s5)

(s1) edge [font=\scriptsize,sloped,above] node { $h \land \neg can$ } (s2)
(s1) edge [font=\scriptsize,sloped,above, bend right=10] node { $c \lor (can \land \neg h)$ } (s5)
(s1) edge [font=\scriptsize,sloped,above, bend right=20] node { $h \land can$ } (s6)

(s2) edge [font=\scriptsize,sloped,above] node { $c \land \neg can$ } (s3)
(s2) edge [font=\scriptsize,sloped,above, bend right=10] node { $can$ } (s6)

(s3) edge [font=\scriptsize,sloped,above, bend left=10] node { $ h \lor can $ } (s6)

(s4) edge [font=\scriptsize,sloped,above] node { $a \lor b$ } (s5)

(s5) edge [font=\scriptsize,sloped,above, bend right=20] node { $h$ } (s6)
;
\end{tikzpicture}}
    \caption{FSA for the composite task. The LTL formula is $(\Event((a \lor b) \land \Event(c \land \Event h)) \land \Always ! can) \lor (\Event((a \lor b) \land \Event h) \land \Event can) \land \Always ! o$. The natural language interpretation is ``Deliver package $a$ or $b$, and then $c$, unless $c$ gets cancelled, and then return to home $h$. And always avoid obstacles''.}
    \label{fig:fsa-composite}
\end{figure}

\pagebreak

\subsection{Full Experimental Results}\label{sec:appendix-results}

For the satisfaction experiments for the delivery domain, 10 policies were trained for each task and for each baseline. Training was done for 1600 episodes, with 100 steps per episode. Every 2000 training steps, the policies were tested on the domain and the returns recorded. For this discrete domain, we know the minimum and maximum possible returns for each task, and we normalized the returns using these minimum and maximum returns. The error bars are the standard deviation of the returns over the 10 policies' rollouts.

For the satisfaction experiments for the reacher domain, a single policy was trained for each task and for each baseline. The baselines were trained for 900 epochs, with 50 steps per epoch. Every 2500 training steps, each policy was tested by doing 10 rollouts and recording the returns. For the \texttt{RM} baseline, training was for 1000 epochs with 800 steps per epoch, and the policy was tested every 8000 training steps. Because we don't know the minimum and maximum rewards for each task, we did not normalize the returns. The error bars are the standard deviation over the 10 rollouts for each baseline.

For the composability experiments, a set of options was trained once, and then meta-policing training using \texttt{LOF-VI}, \texttt{LOF-QL}, and \texttt{Greedy} was done for each task. Returns were recorded at every training step by rolling out each baseline 10 times. The error bars are the standard deviations on the 10 rollouts.

For the pick-and-place domain, 1 policy was trained for the satisfaction experiments, and experimental results were evaluated over 2 rollouts. Training was done for 7500 epochs with 1000 steps per epoch. Every 250,000 training steps, the policy was tested by doing 2 rollouts and recording the returns. For the composability experiments, returns were recorded by rolling out each baseline 2 times. The \texttt{RM} baseline was trained over 10000 epochs with 1000 steps per epoch.

Code and videos of the domains and tasks are in the supplement.

\begin{figure*}[!th]
\centering
\begin{subfigure}[b]{0.3\textwidth}
  \centering
  \includegraphics[width=3cm]{figures/delivery_environment-cropped.pdf}
  \caption{Delivery domain.}
  \label{fig:appendix-discrete-domain}
\end{subfigure} \hfill
\begin{subfigure}[b]{.3\textwidth}
  \centering
  \includegraphics[width=4cm]{figures/discrete/satisfaction/results_averaged_over_tasks.png}
  \caption{Averaged.}
  \label{fig:appendix-discrete-satisfaction-average}
 \end{subfigure} \hfill
\begin{subfigure}[b]{.3\textwidth}
  \centering
  \includegraphics[width=4cm]{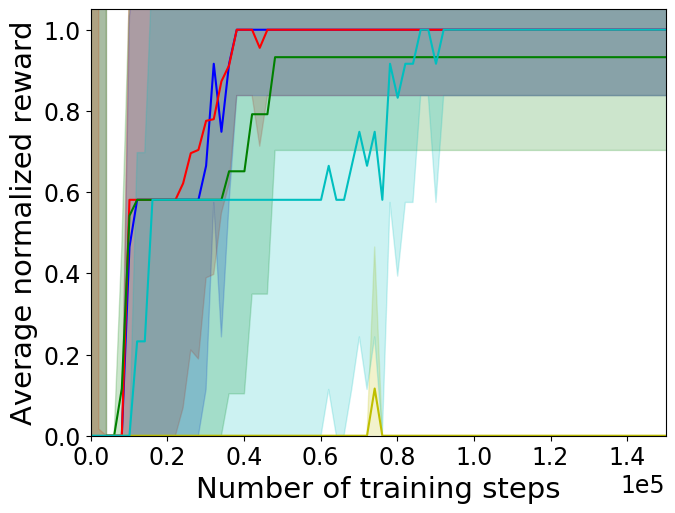}
  \caption{Composite.}
  \label{fig:appendix-discrete-satisfaction-composite}
\end{subfigure} \hfill
\begin{subfigure}[b]{.3\textwidth}
  \centering
  \includegraphics[width=4cm]{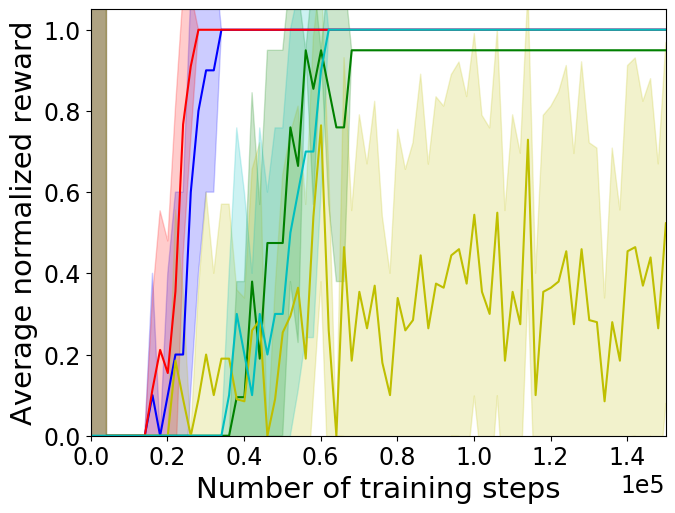}
  \caption{OR.}
  \label{fig:appendix-discrete-satisfaction-or}
\end{subfigure} \hfill
\begin{subfigure}[b]{.3\textwidth}
  \centering
  \includegraphics[width=4cm]{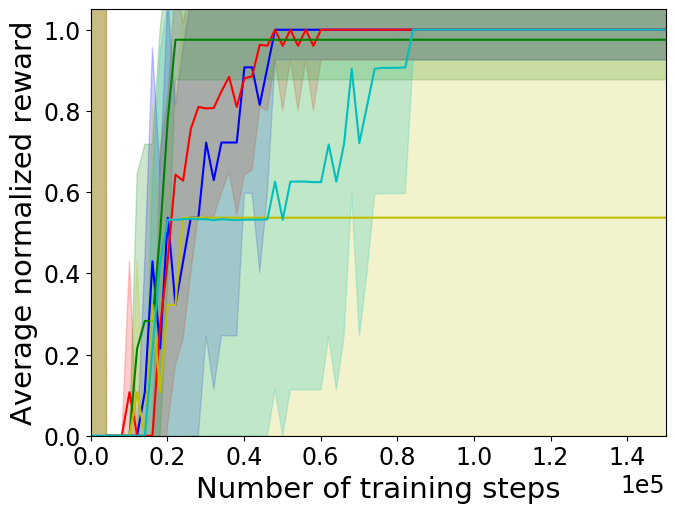}
  \caption{IF.}
  \label{fig:appendix-discrete-satisfaction-if}
\end{subfigure} \hfill
\begin{subfigure}[b]{.3\textwidth}
  \centering
  \includegraphics[width=4cm]{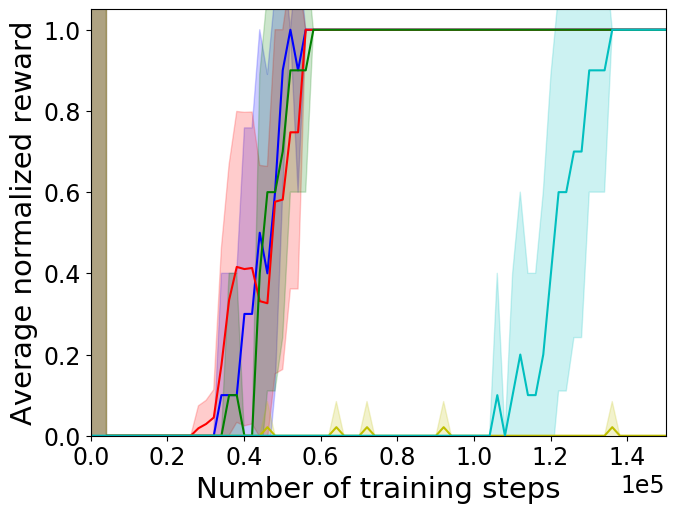}
  \caption{Sequential.}
  \label{fig:appendix-discrete-satisfaction-sequential}
\end{subfigure}

\begin{subfigure}[b]{\textwidth}
  \centering
  \includegraphics[height=0.6cm]{figures/discrete/satisfaction/legend.png}
  \label{fig:appendix-discrete-satisfaction-legend}
\end{subfigure}

\caption{All satisfaction experiments on the delivery domain. Notice how for the composite and OR tasks (Figs.~\ref{fig:appendix-discrete-satisfaction-composite} and~\ref{fig:appendix-discrete-satisfaction-or}), the \texttt{Greedy} baseline plateaus before \texttt{LOF-VI} and \texttt{LOF-QL}. This is because \texttt{Greedy} chooses a suboptimal path through the FSA, whereas \texttt{LOF-VI} and \texttt{LOF-QL} find an optimal path. Also, notice that \texttt{RM} takes many more training steps to achieve the optimal cumulative reward. This is because for \texttt{RM}, the only reward signal is from reaching the goal state. It takes a long time for the agent to learn an optimal policy from such a sparse reward signal. This is particularly evident for the sequential task (Fig.~\ref{fig:appendix-discrete-satisfaction-sequential}), which requires the agent to take a longer sequence of actions/FSA states before reaching the goal. The options-based algorithms train much faster because when training the options, the agent receives a reward for reaching each subgoal, and therefore the reward signal is much richer.}
\label{fig:appendix-discrete-satisfaction-experiments}
\end{figure*}

\begin{figure*}[!th]
\centering
\begin{subfigure}[b]{.3\textwidth}
  \centering
  \includegraphics[width=4cm]{figures/reacher_environment-cropped.pdf}
  \caption{Reacher domain.}
  \label{fig:appendix-continuous-domain}
\end{subfigure} \hfill
\begin{subfigure}[b]{.3\textwidth}
  \centering
  \includegraphics[width=4cm]{figures/continuous/satisfaction/results_averaged_over_tasks.png}
  \caption{Averaged.}
  \label{fig:appendix-continuous-satisfaction-average}
 \end{subfigure} \hfill
\begin{subfigure}[b]{.3\textwidth}
  \centering
  \includegraphics[width=4cm]{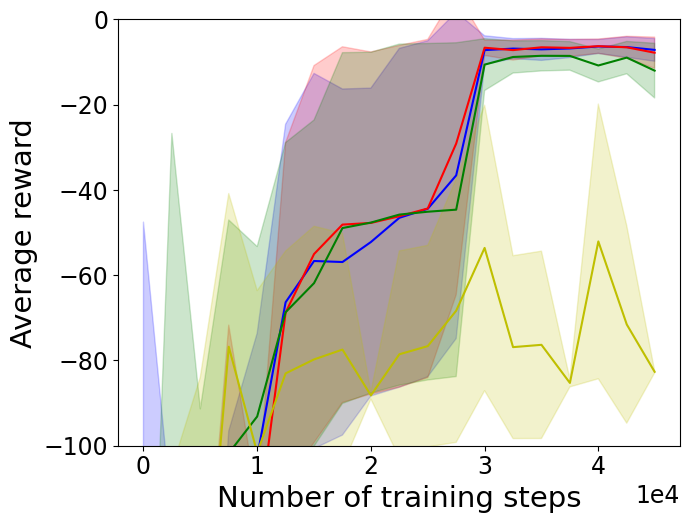}
  \caption{Composite.}
  \label{fig:appendix-continuous-satisfaction-composite}
\end{subfigure} \hfill
\begin{subfigure}[b]{.3\textwidth}
  \centering
  \includegraphics[width=4cm]{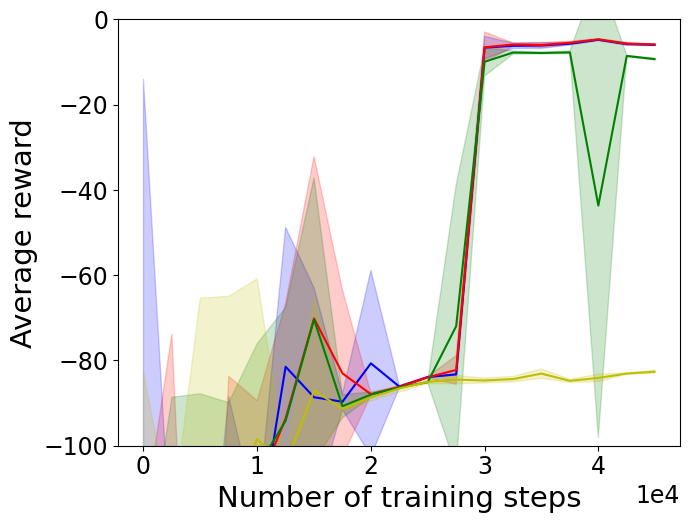}
  \caption{OR.}
  \label{fig:appendix-continuous-satisfaction-or}
\end{subfigure} \hfill
\begin{subfigure}[b]{.3\textwidth}
  \centering
  \includegraphics[width=4cm]{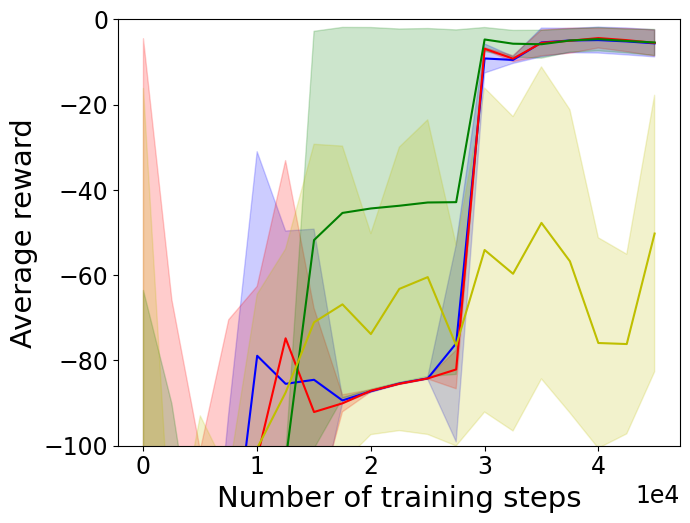}
  \caption{IF.}
  \label{fig:appendix-continuous-satisfaction-if}
\end{subfigure} \hfill
\begin{subfigure}[b]{.3\textwidth}
  \centering
  \includegraphics[width=4cm]{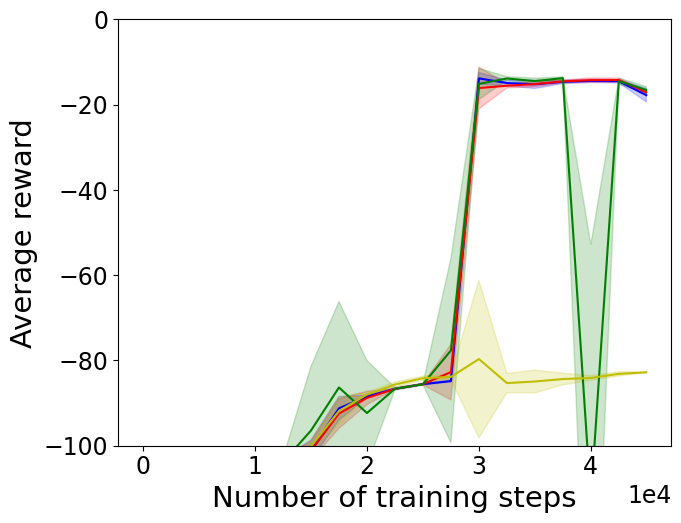}
  \caption{Sequential.}
  \label{fig:appendix-continuous-satisfaction-sequential}
\end{subfigure}

\begin{subfigure}[b]{\textwidth}
  \centering
  \includegraphics[height=0.6cm]{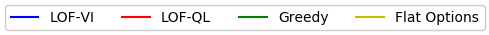}
  \label{fig:appendix-continuous-satisfaction-legend}
\end{subfigure}
\caption{Satisfaction experiments for the reacher domain, without \texttt{RM} results. The results are equivalent to the results on the delivery domain.}
\label{fig:appendix-continuous-satifaction-experiments}
\end{figure*}

\begin{figure*}[!th]
\centering
\begin{subfigure}[b]{.3\textwidth}
  \centering
  \includegraphics[width=4cm]{figures/reacher_environment-cropped.pdf}
  \caption{Reacher domain.}
  \label{fig:appendix-continuous-domain-with-rm}
\end{subfigure} \hfill
\begin{subfigure}[b]{.3\textwidth}
  \centering
  \includegraphics[width=4cm]{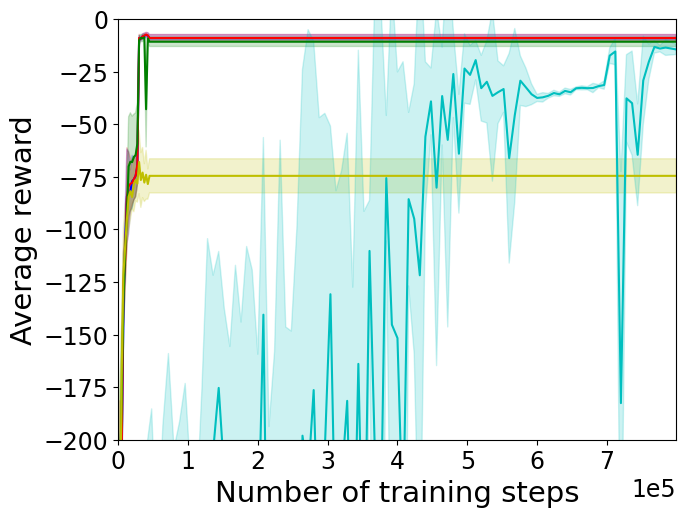}
  \caption{Averaged.}
  \label{fig:appendix-continuous-satisfaction-average-with-rm}
 \end{subfigure} \hfill
\begin{subfigure}[b]{.3\textwidth}
  \centering
  \includegraphics[width=4cm]{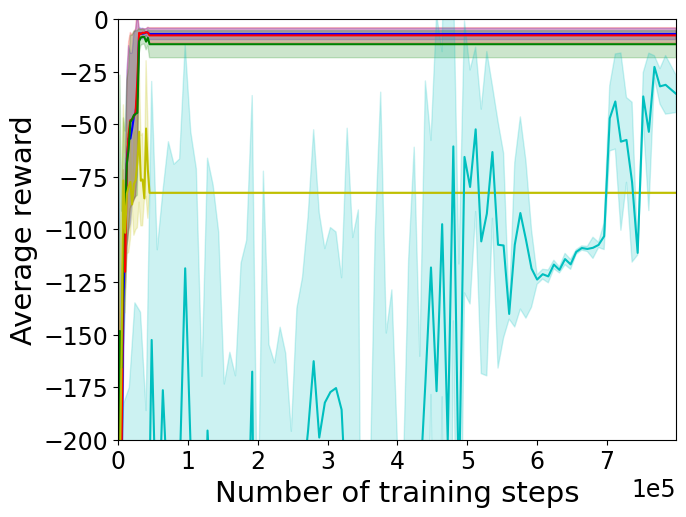}
  \caption{Composite.}
  \label{fig:appendix-continuous-satisfaction-composite-with-rm}
\end{subfigure} \hfill
\begin{subfigure}[b]{.3\textwidth}
  \centering
  \includegraphics[width=4cm]{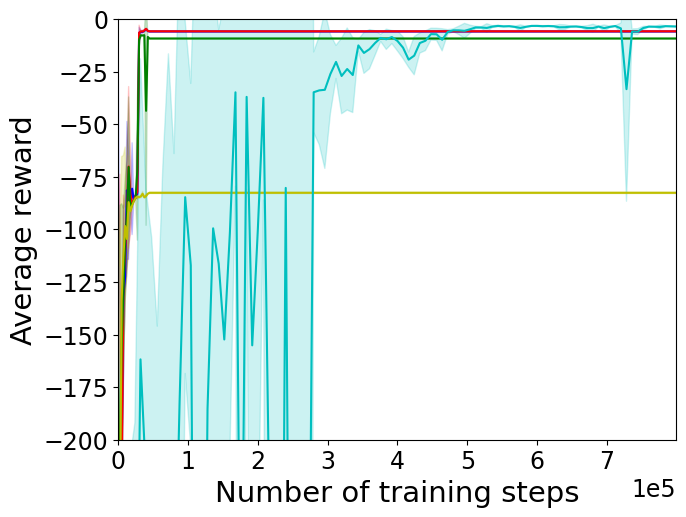}
  \caption{OR.}
  \label{fig:appendix-continuous-satisfaction-or-with-rm}
\end{subfigure} \hfill
\begin{subfigure}[b]{.3\textwidth}
  \centering
  \includegraphics[width=4cm]{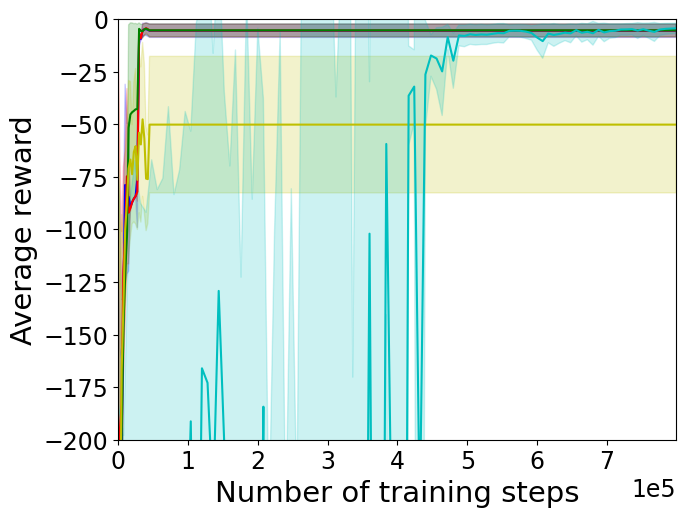}
  \caption{IF.}
  \label{fig:appendix-continuous-satisfaction-if-with-rm}
\end{subfigure} \hfill
\begin{subfigure}[b]{.3\textwidth}
  \centering
  \includegraphics[width=4cm]{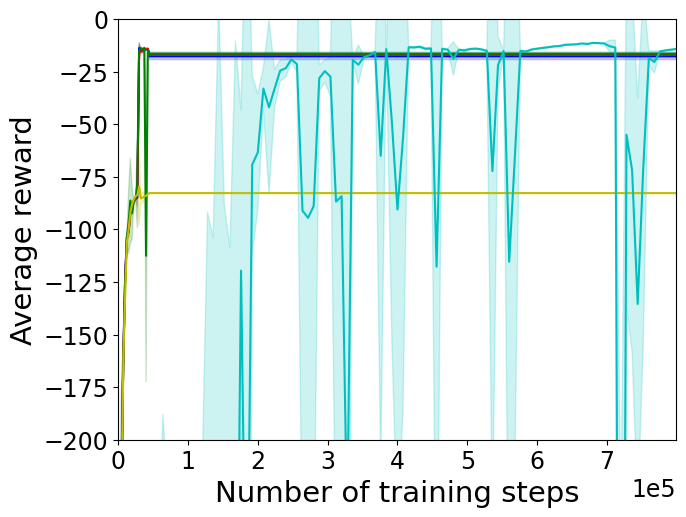}
  \caption{Sequential.}
  \label{fig:appendix-continuous-satisfaction-sequential-with-rm}
\end{subfigure}

\begin{subfigure}[b]{\textwidth}
  \centering
  \includegraphics[height=0.6cm]{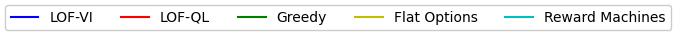}
  \label{fig:appendix-continuous-satisfaction-legend-with-rm}
\end{subfigure}
\caption{Satisfaction experiments for the reacher domain, including \texttt{RM} results. \texttt{RM} takes significantly more training steps to train than the other baselines, although it eventually reaches and surpasses the cumulative reward of the other baselines. This is because for the continuous domain, we violate some of the conditions required for optimality when using the Logical Options Framework -- in particular, the condition that each subgoal is associated with a single state. In a continuous environment, this condition is impossible to meet, and therefore we made the subgoals small spherical regions, and we only made the subgoals associated with specific Cartesian coordinates and not velocities (which are also in the state space). Meanwhile, the optimality conditions of \texttt{RM} are looser and were not violated, which is why it achieves a higher final cumulative reward.}
\label{fig:appendix-continuous-satifaction-experiments-with-rm}
\end{figure*}

\begin{figure*}[!th]
\centering
\begin{subfigure}[b]{.3\textwidth}
  \centering
  \includegraphics[width=3.4cm]{figures/arm_domain-cropped.pdf}
  \caption{Reacher domain.}
  \label{fig:appendix-arm-domain}
\end{subfigure} \hfill
\begin{subfigure}[b]{.3\textwidth}
  \centering
  \includegraphics[width=4cm]{figures/arm/satisfaction/results_averaged_over_tasks.png}
  \caption{Averaged.}
  \label{fig:appendix-arm-satisfaction-average}
 \end{subfigure} \hfill
\begin{subfigure}[b]{.3\textwidth}
  \centering
  \includegraphics[width=4cm]{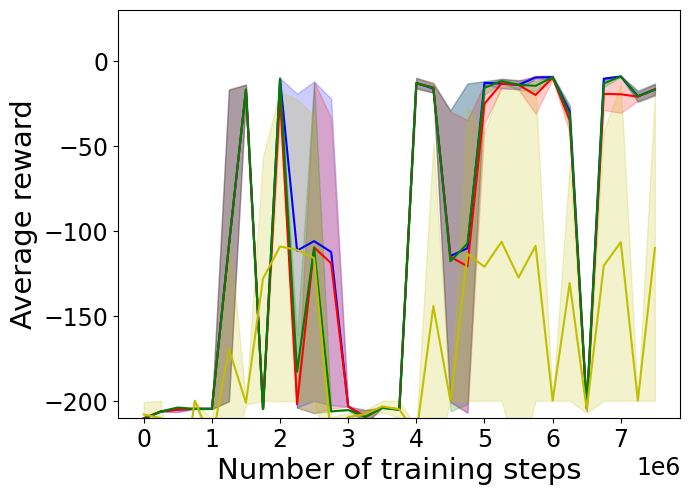}
  \caption{Composite.}
  \label{fig:appendix-arm-satisfaction-composite}
\end{subfigure} \hfill
\begin{subfigure}[b]{.3\textwidth}
  \centering
  \includegraphics[width=4cm]{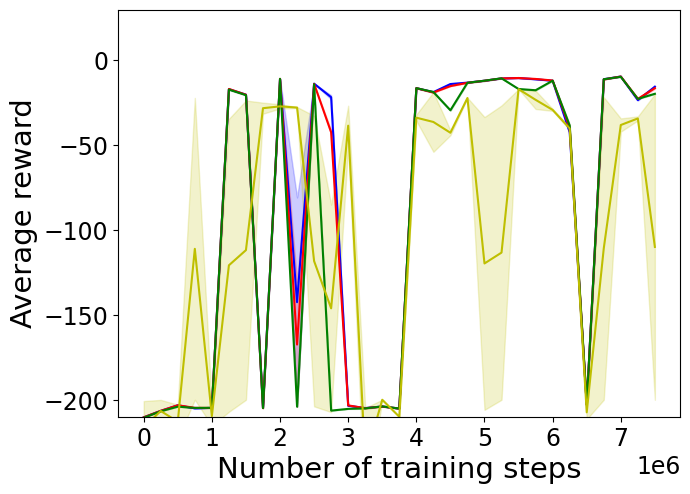}
  \caption{OR.}
  \label{fig:appendix-arm-satisfaction-or}
\end{subfigure} \hfill
\begin{subfigure}[b]{.3\textwidth}
  \centering
  \includegraphics[width=4cm]{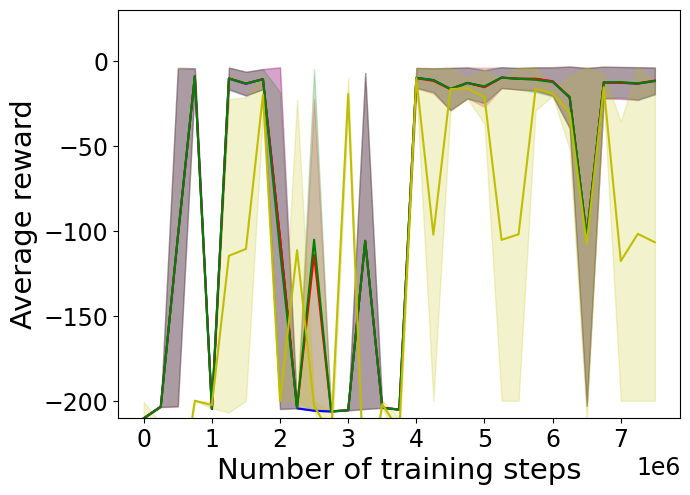}
  \caption{IF.}
  \label{fig:appendix-arm-satisfaction-if}
\end{subfigure} \hfill
\begin{subfigure}[b]{.3\textwidth}
  \centering
  \includegraphics[width=4cm]{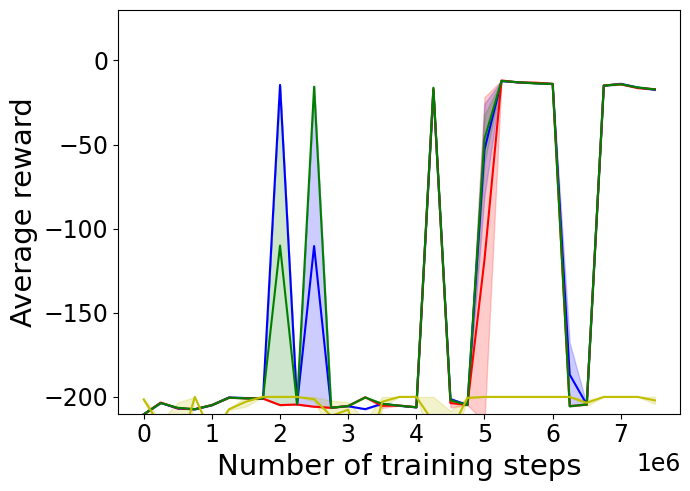}
  \caption{Sequential.}
  \label{fig:appendix-arm-satisfaction-sequential}
\end{subfigure}

\begin{subfigure}[b]{\textwidth}
  \centering
  \includegraphics[height=0.6cm]{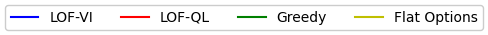}
  \label{fig:appendix-arm-satisfaction-legend}
\end{subfigure}
\caption{Satisfaction experiments for the pick-and-place domain, without \texttt{RM} results. The results are equivalent to the results on the delivery and reacher domains.}
\label{fig:appendix-arm-satifaction-experiments}
\end{figure*}

\begin{figure*}[!th]
\centering
\begin{subfigure}[b]{.3\textwidth}
  \centering
  \includegraphics[width=3.4cm]{figures/arm_domain-cropped.pdf}
  \caption{Reacher domain.}
  \label{fig:appendix-arm-domain-with-rm}
\end{subfigure} \hfill
\begin{subfigure}[b]{.3\textwidth}
  \centering
  \includegraphics[width=4cm]{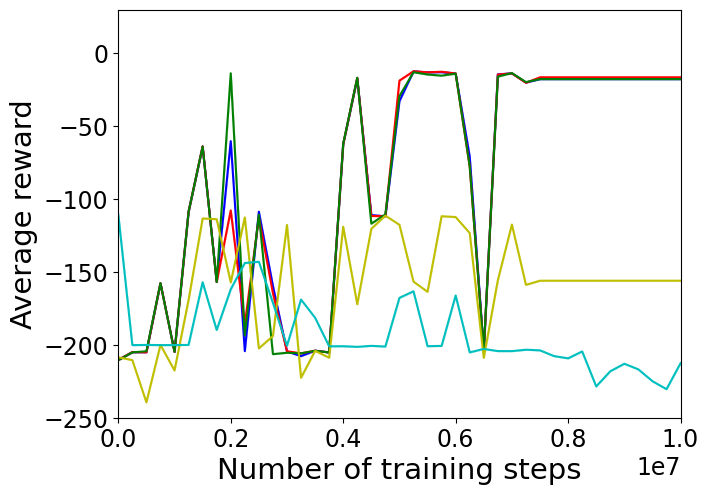}
  \caption{Averaged.}
  \label{fig:appendix-arm-satisfaction-average-with-rm}
 \end{subfigure} \hfill
\begin{subfigure}[b]{.3\textwidth}
  \centering
  \includegraphics[width=4cm]{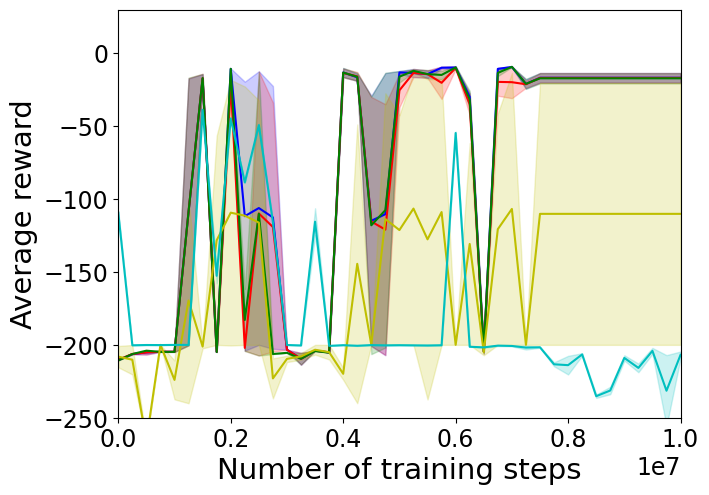}
  \caption{Composite.}
  \label{fig:appendix-arm-satisfaction-composite-with-rm}
\end{subfigure} \hfill
\begin{subfigure}[b]{.3\textwidth}
  \centering
  \includegraphics[width=4cm]{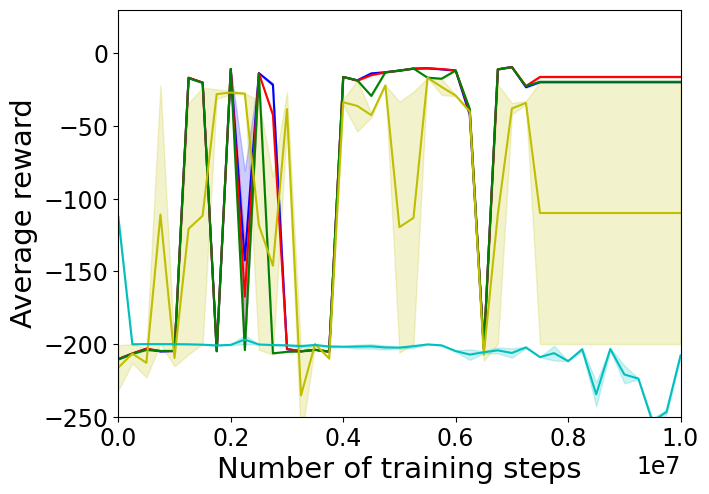}
  \caption{OR.}
  \label{fig:appendix-arm-satisfaction-or-with-rm}
\end{subfigure} \hfill
\begin{subfigure}[b]{.3\textwidth}
  \centering
  \includegraphics[width=4cm]{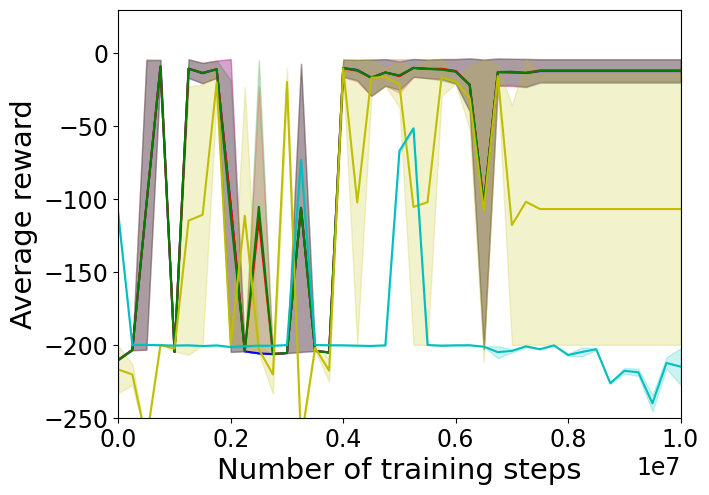}
  \caption{IF.}
  \label{fig:appendix-arm-satisfaction-if-with-rm}
\end{subfigure} \hfill
\begin{subfigure}[b]{.3\textwidth}
  \centering
  \includegraphics[width=4cm]{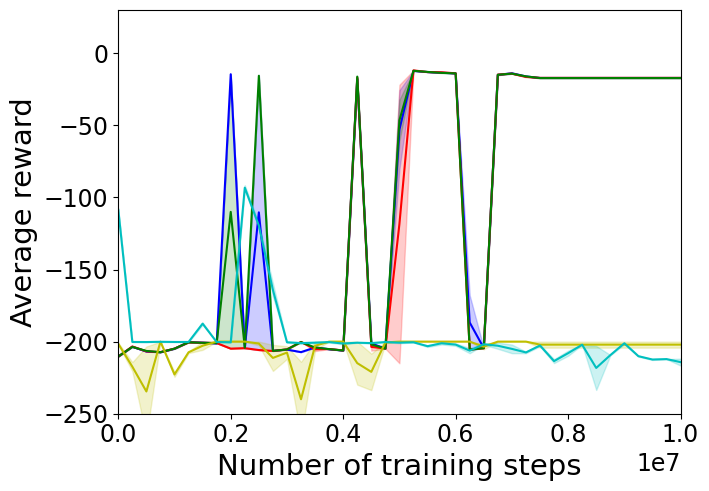}
  \caption{Sequential.}
  \label{fig:appendix-arm-satisfaction-sequential-with-rm}
\end{subfigure}

\begin{subfigure}[b]{\textwidth}
  \centering
  \includegraphics[height=0.6cm]{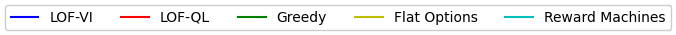}
  \label{fig:appendix-arm-satisfaction-legend-with-rm}
\end{subfigure}
\caption{Satisfaction experiments for the pick-and-place domain, including \texttt{RM} results. For the pick-and-place domain, \texttt{RM} did not converge to a solution within the training time allotted for it (10 million training steps).}
\label{fig:appendix-arm-satifaction-experiments-with-rm}
\end{figure*}

\begin{figure*}[!th]
\centering
\begin{subfigure}[b]{0.3\textwidth}
  \centering
  \includegraphics[width=3cm]{figures/delivery_environment-cropped.pdf}
  \caption{Delivery domain.}
  \label{fig:appendix-discrete-domain-composability}
\end{subfigure} \hfill
\begin{subfigure}[b]{.3\textwidth}
  \centering
  \includegraphics[width=4cm]{figures/discrete/composability/results_averaged_over_tasks.png}
  \caption{Averaged.}
  \label{fig:appendix-discrete-composability-average}
 \end{subfigure} \hfill
\begin{subfigure}[b]{.3\textwidth}
  \centering
  \includegraphics[width=4cm]{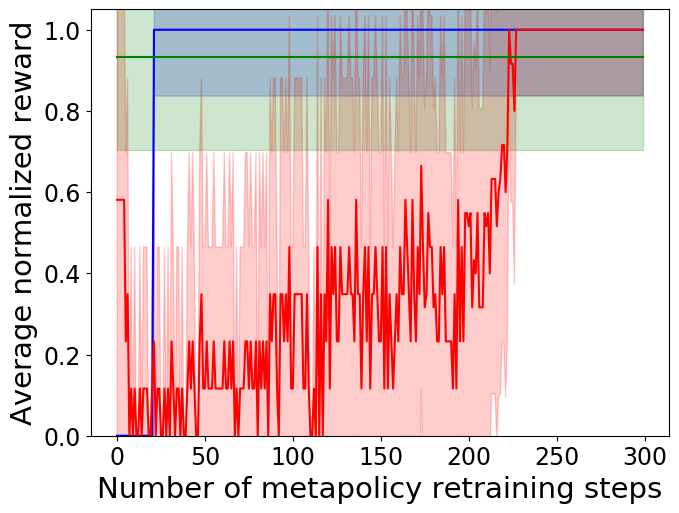}
  \caption{Composite.}
  \label{fig:appendix-discrete-composability-composite}
\end{subfigure} \hfill
\begin{subfigure}[b]{.3\textwidth}
  \centering
  \includegraphics[width=4cm]{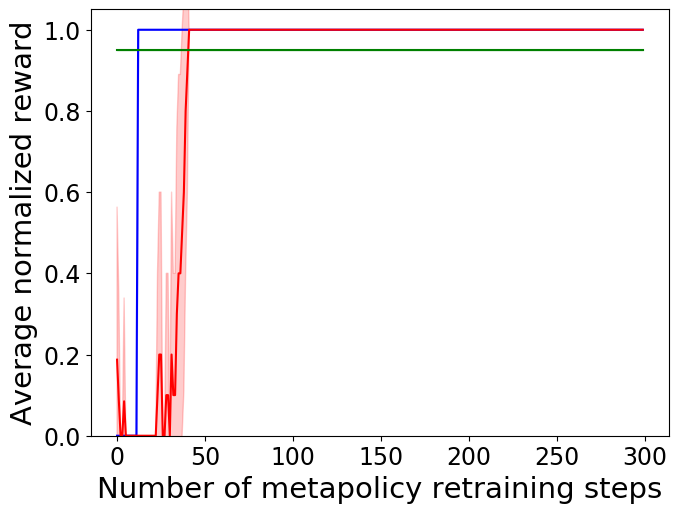}
  \caption{OR.}
  \label{fig:appendix-discrete-composability-or}
\end{subfigure} \hfill
\begin{subfigure}[b]{.3\textwidth}
  \centering
  \includegraphics[width=4cm]{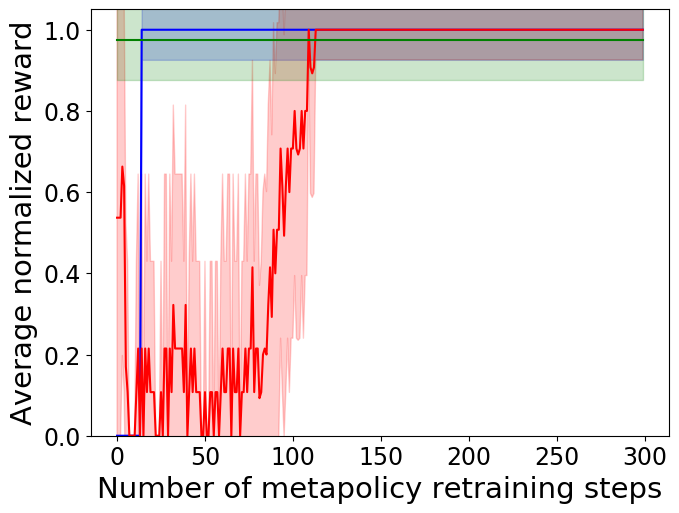}
  \caption{IF.}
  \label{fig:appendix-discrete-composability-if}
\end{subfigure} \hfill
\begin{subfigure}[b]{.3\textwidth}
  \centering
  \includegraphics[width=4cm]{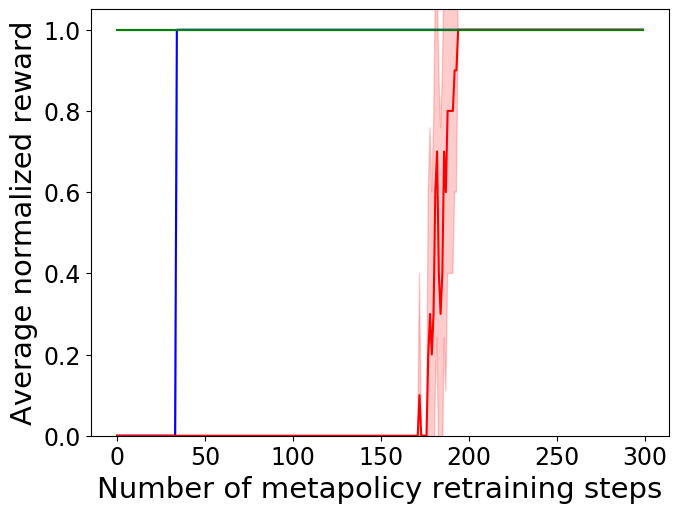}
  \caption{Sequential.}
  \label{fig:appendix-discrete-composability-sequential}
\end{subfigure}

\begin{subfigure}[b]{\textwidth}
  \centering
  \includegraphics[height=0.6cm]{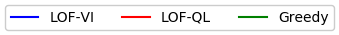}
  \label{fig:appendix-discrete-composability-legend}
\end{subfigure}

\caption{All composability experiments for the delivery domain.}
\label{fig:appendix-discrete-composability-experiments}
\end{figure*}

\begin{figure*}[!th]
\centering
\begin{subfigure}[b]{0.3\textwidth}
  \centering
  \includegraphics[width=4cm]{figures/reacher_environment-cropped.pdf}
  \caption{Delivery domain.}
  \label{fig:appendix-continuous-domain-composablity}
\end{subfigure} \hfill
\begin{subfigure}[b]{.3\textwidth}
  \centering
  \includegraphics[width=4cm]{figures/continuous/composability/results_averaged_over_tasks.png}
  \caption{Averaged.}
  \label{fig:appendix-continuous-composability-average}
 \end{subfigure} \hfill
\begin{subfigure}[b]{.3\textwidth}
  \centering
  \includegraphics[width=4cm]{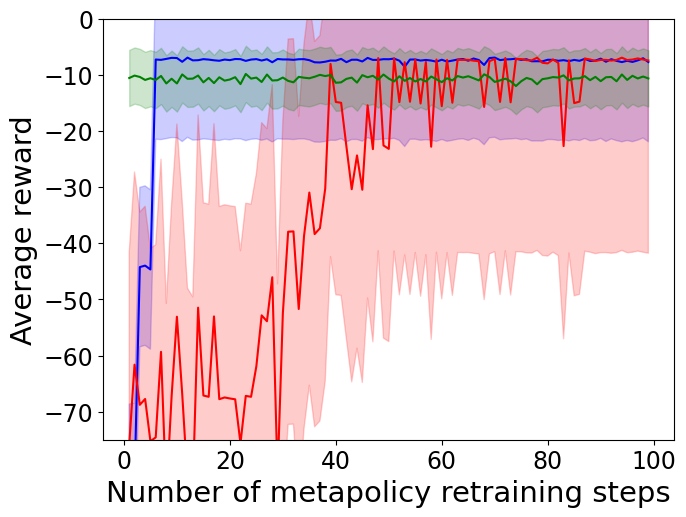}
  \caption{Composite.}
  \label{fig:appendix-continuous-composability-composite}
\end{subfigure} \hfill
\begin{subfigure}[b]{.3\textwidth}
  \centering
  \includegraphics[width=4cm]{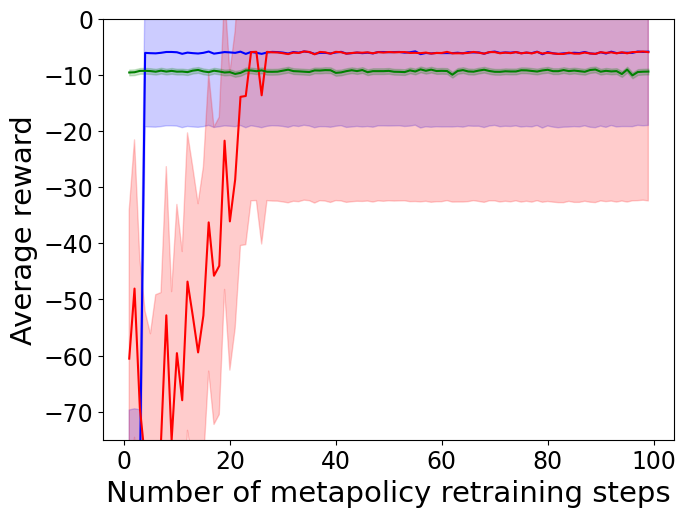}
  \caption{OR.}
  \label{fig:appendix-continuous-composability-or}
\end{subfigure} \hfill
\begin{subfigure}[b]{.3\textwidth}
  \centering
  \includegraphics[width=4cm]{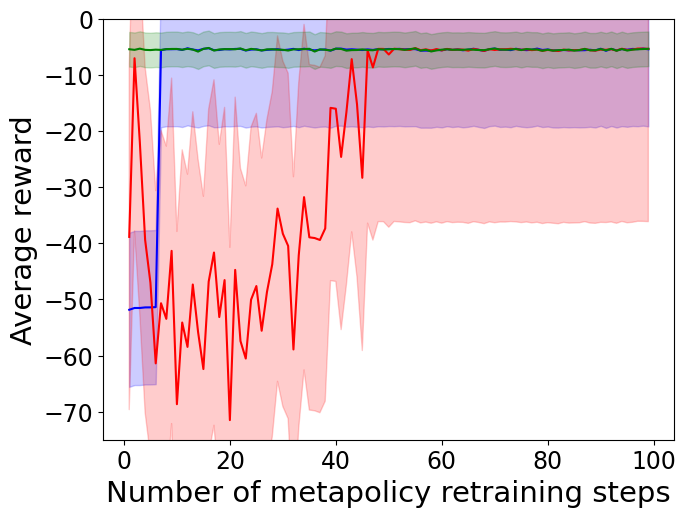}
  \caption{IF.}
  \label{fig:appendix-continuous-composability-if}
\end{subfigure} \hfill
\begin{subfigure}[b]{.3\textwidth}
  \centering
  \includegraphics[width=4cm]{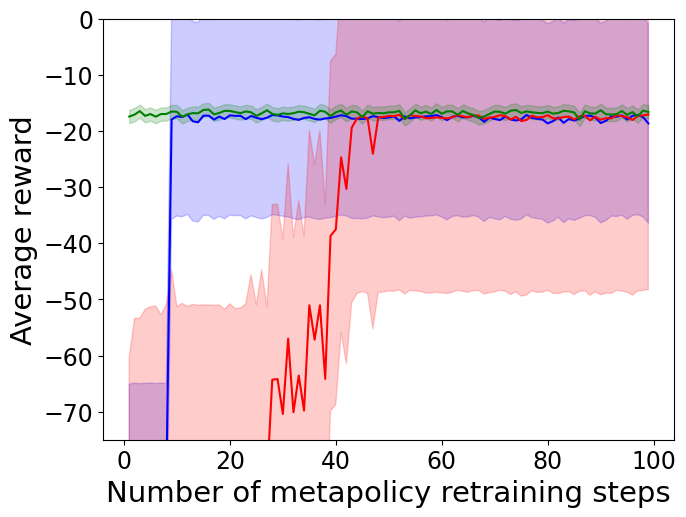}
  \caption{Sequential.}
  \label{fig:appendix-continuous-composability-sequential}
\end{subfigure}

\begin{subfigure}[b]{\textwidth}
  \centering
  \includegraphics[height=0.6cm]{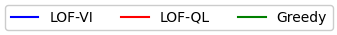}
  \label{fig:appendix-continuous-composability-legend}
\end{subfigure}

\caption{All composability experiments for the reacher domain.}
\label{fig:appendix-continuous-composability-experiments}
\end{figure*}

\begin{figure*}[!th]
\centering
\begin{subfigure}[b]{0.3\textwidth}
  \centering
  \includegraphics[width=3.4cm]{figures/arm_domain-cropped.pdf}
  \caption{Delivery domain.}
  \label{fig:appendix-arm-domain-composablity}
\end{subfigure} \hfill
\begin{subfigure}[b]{.3\textwidth}
  \centering
  \includegraphics[width=4cm]{figures/arm/composability/results_averaged_over_tasks.png}
  \caption{Averaged.}
  \label{fig:appendix-arm-composability-average}
 \end{subfigure} \hfill
\begin{subfigure}[b]{.3\textwidth}
  \centering
  \includegraphics[width=4cm]{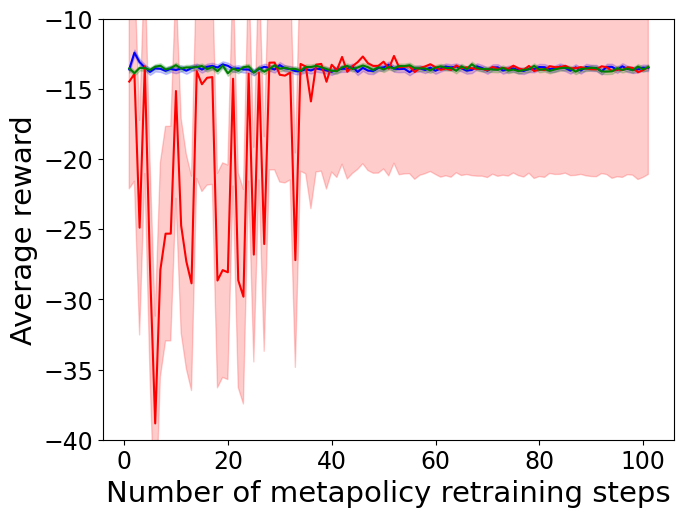}
  \caption{Composite.}
  \label{fig:appendix-arm-composability-composite}
\end{subfigure} \hfill
\begin{subfigure}[b]{.3\textwidth}
  \centering
  \includegraphics[width=4cm]{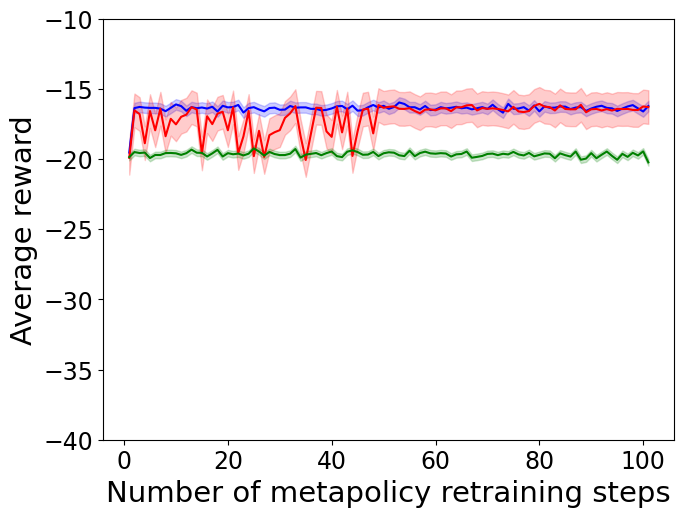}
  \caption{OR.}
  \label{fig:appendix-arm-composability-or}
\end{subfigure} \hfill
\begin{subfigure}[b]{.3\textwidth}
  \centering
  \includegraphics[width=4cm]{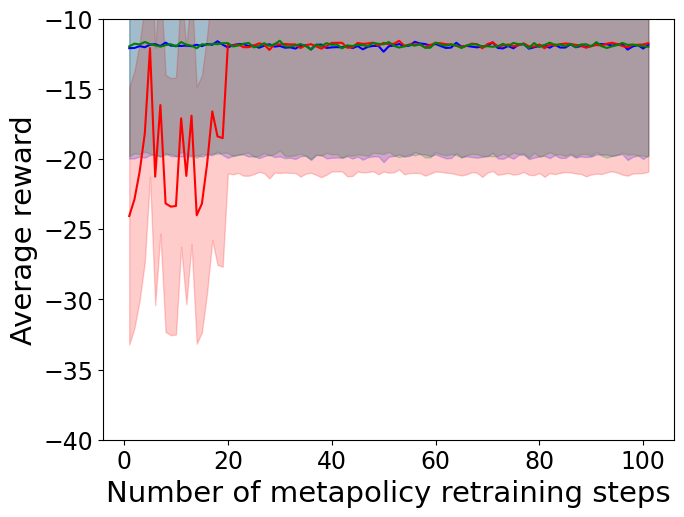}
  \caption{IF.}
  \label{fig:appendix-arm-composability-if}
\end{subfigure} \hfill
\begin{subfigure}[b]{.3\textwidth}
  \centering
  \includegraphics[width=4cm]{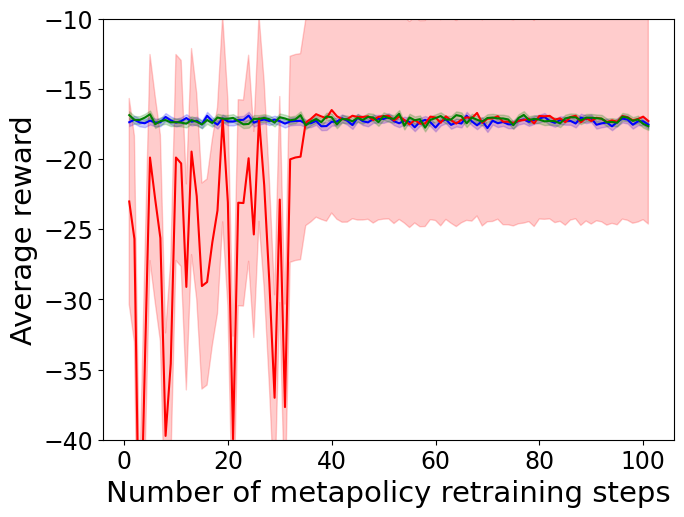}
  \caption{Sequential.}
  \label{fig:appendix-arm-composability-sequential}
\end{subfigure}

\begin{subfigure}[b]{\textwidth}
  \centering
  \includegraphics[height=0.6cm]{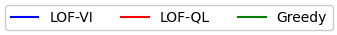}
  \label{fig:appendix-arm-composability-legend}
\end{subfigure}

\caption{All composability experiments for the pick-and-place domain.}
\label{fig:appendix-arm-composability-experiments}
\end{figure*}

\section{Further Discussion}

\paragraph{What happens when incorrect rules are used?} One benefit of representing the rules of the environment as LTL formulae/automata is that these forms of representing rules are much more interpretable than alternatives (such as neural nets). Therefore, if an agent's learned policy has bad behavior, a user of LOF can inspect the rules to see if the bad behavior is a consequence of a bad rule specification. Furthermore, one of the consequences of composability is that any modifications to the FSA will alter the resulting policy in a direct and predictable way. Therefore, for example, if an incorrect human-specified task yields undesirable behavior, with our framework it is possible to tweak the task and test the new policy without any additional low-level training (however, tweaking the safety rules would require retraining the logical options).

\paragraph{What happens if there is a rule conflict?} If the specified LTL formula is invalid, the LTL-to-automaton translation tool will either throw an error or return a trivial single-state automaton that is not an accepting state. Rollouts would terminate immediately.

\paragraph{What happens if the agent can't satisfy a task without violating a rule?} The solution to this problem depends on the user's priorities. In our formulation, we have assigned finite costs to rule violations and an infinite cost to not satisfying the task (see Appendix~\ref{sec:appendix-proofs}). We have prioritized task satisfaction over safety satisfaction. However, it is possible to flip the priorities around by terminating training/rollouts if there is a safety violation. In our proofs, we have assumed that the agent can reach every subgoal from any state, implying either that it is always possible to avoid safety violations or that safety violations are allowed.

\paragraph{Why is the safety property not composable?} The safety property is not composable because we allow safety propositions to be associated with more than one state in the environment (unlike subgoals). The fact that there can be multiple instances of a safety proposition in the environment means that it is impossible to guarantee that a new option policy will be optimal if retraining is done only at the level of the safety automaton and not also over the low-level states. In order to guarantee optimality, retraining would have to be done over both the high and low levels (the safety automaton and the environment). Our definition of composability involves only replanning over the high level of the FSA. Therefore, safety properties are not composable. Furthermore, rewards/costs of the safety property can be associated with propositions and not just with states (as with the liveness property). This is because a safety violation via one safety proposition (e.g., a car going onto the wrong side of the road) may incur a different penalty than a violation via a different proposition (a car going off the road). The propositions are associated with low-level states of the environment. Therefore any retraining would have to involve retraining at both the high and low levels, once again violating our definition of composability.


\paragraph{Simplifying the option transition model:} In our experiments, we simplify the transition model by setting $\gamma = 1$, an assumption that does not affect convergence to optimality. In the case where $\gamma = 1$, Eq.~\ref{eq:option-transitions} reduces to $T_o(s' \vert s) = \sum_k p(s', k)$. Assuming that the option terminates only at state $s_g$, then Eq.\ref{eq:option-transitions} further reduces to $T_o(s_g | s) = 1$ and $T_o(s' | s) = 0$ for all other $s' \neq s_g$. Therefore no learning is required for the transition model. For cases where the assumption that $\gamma = 1$ does not apply, \citep{abel2019expected} contains an interesting discussion. 

\paragraph{Learning the option reward model:} The option reward model $R_o(s)$ is the expected reward of carrying out option $o$ to termination from state $s$. It is equivalent to a value function. Therefore, it is convenient if the policy-learning algorithm used to learn the options learns a value function as well as a policy (e.g., Q-learning and PPO). However, as long as the expected return can be computed between pairs of states, it is not necessary to learn a complete value function. This is because during Logical Value Iteration, the reward model is only queried at discrete points in the state space (typically corresponding to the initial state and the subgoals). So as long as expected returns between the initial state and subgoals can be computed, Logical Value Iteration will work.

\paragraph{Why is \texttt{LOF-VI} so much more efficient than the \texttt{RM} baseline?} In short, \texttt{LOF-VI} is more efficient than \texttt{RM} because \texttt{LOF-VI} has a dense reward function during training and \texttt{RM} has a sparse reward function. During training, \texttt{LOF-VI} trains the options independently and rewards the agent for reaching the subgoals associated with the options. This is in effect a dense reward function. The generic reward function for \texttt{RM} only rewards the agent for reaching the goal state. There are no other high-level rewards to guide the agent through the task. This is a very sparse reward that results in less efficient training. \texttt{RM}'s reward function could easily be made dense by rewarding every transition of the automaton. In this case, \texttt{RM} would probably train as efficiently as \texttt{LOF-VI}. However, imagine an FSA with two paths to the goal state. One path has only 1 transition but has much lower low-level cost, and one path has 20 transitions and a much higher low-level cost. \texttt{RM} might learn to prefer the reward-heavy 20-transition path rather than the reward-light 1-transition path, even if the 1-transition path results in a lower low-level cost. In theory it might be possible to design an \texttt{RM} reward function that adjusts the automaton transition reward depending on the length of the path that the state is in, but this would not be a trivial task when accounting for branching and merging paths. We therefore decided that it would be a fairer comparison to use a trivial \texttt{RM} reward function, just as we use a trivial reward function for the LOF baselines. However, we were careful to not list increased efficiency in our list of contributions; although increased efficiency was an observed side effect of LOF, LOF is not inherently more efficient than other algorithms besides the fact that it automatically imposes a dense reward on reaching subgoals.


\end{document}